\documentclass{article}

\usepackage{times}
\usepackage{graphicx} 
\usepackage{subfigure}

\usepackage{natbib}

\usepackage{algorithm}
\usepackage{algorithmic}
\usepackage[algo2e,linesnumbered,lined,boxed,vlined]{algorithm2e}

\usepackage{amsfonts}
\usepackage{amssymb}

\usepackage{hyperref}

\usepackage[accepted]{icml2016}

\usepackage{comment}
\usepackage{amsthm}
\usepackage{mathtools}
\usepackage{tabularx}
\usepackage{multirow}

\newtheorem{proposition}{Proposition}
\newtheorem{corollary}{Corollary}

\def\b{\ensuremath\boldsymbol}

\usepackage{lastpage}
\usepackage{fancyhdr}
\usepackage{url}
\icmltitlerunning{Locally Linear Embedding and its Variants: Tutorial and Survey}

\begin{document}

\AddToShipoutPictureBG*{%
  \AtPageUpperLeft{%
    \setlength\unitlength{1in}%
    \hspace*{\dimexpr0.5\paperwidth\relax}
    \makebox(0,-0.75)[c]{\normalsize {\color{black} To appear as a part of an upcoming textbook on dimensionality reduction and manifold learning.}}
    }}

\twocolumn[
\icmltitle{Locally Linear Embedding and its Variants: Tutorial and Survey}

\icmlauthor{Benyamin Ghojogh}{bghojogh@uwaterloo.ca}
\icmladdress{Department of Electrical and Computer Engineering, 
\\Machine Learning Laboratory, University of Waterloo, Waterloo, ON, Canada}
\icmlauthor{Ali Ghodsi}{ali.ghodsi@uwaterloo.ca}
\icmladdress{Department of Statistics and Actuarial Science \& David R. Cheriton School of Computer Science, 
\\Data Analytics Laboratory, University of Waterloo, Waterloo, ON, Canada}
\icmlauthor{Fakhri Karray}{karray@uwaterloo.ca}
\icmladdress{Department of Electrical and Computer Engineering, 
\\Centre for Pattern Analysis and Machine Intelligence, University of Waterloo, Waterloo, ON, Canada}
\icmlauthor{Mark Crowley}{mcrowley@uwaterloo.ca}
\icmladdress{Department of Electrical and Computer Engineering, 
\\Machine Learning Laboratory, University of Waterloo, Waterloo, ON, Canada}

\icmlkeywords{Tutorial, Locally Linear Embedding}

\vskip 0.3in
]

\begin{abstract}
This is a tutorial and survey paper for Locally Linear Embedding (LLE) and its variants. The idea of LLE is fitting the local structure of manifold in the embedding space. In this paper, we first cover LLE, kernel LLE, inverse LLE, and feature fusion with LLE. Then, we cover out-of-sample embedding using linear reconstruction, eigenfunctions, and kernel mapping. Incremental LLE is explained for embedding streaming data. Landmark LLE methods using the Nystrom approximation and locally linear landmarks are explained for big data embedding. We introduce the methods for parameter selection of number of neighbors using residual variance, Procrustes statistics, preservation neighborhood error, and local neighborhood selection. Afterwards, Supervised LLE (SLLE), enhanced SLLE, SLLE projection, probabilistic SLLE, supervised guided LLE (using Hilbert-Schmidt independence criterion), and semi-supervised LLE are explained for supervised and semi-supervised embedding. Robust LLE methods using least squares problem and penalty functions are also introduced for embedding in the presence of outliers and noise. Then, we introduce fusion of LLE with other manifold learning methods including Isomap (i.e., ISOLLE), principal component analysis, Fisher discriminant analysis, discriminant LLE, and Isotop. Finally, we explain weighted LLE in which the distances, reconstruction weights, or the embeddings are adjusted for better embedding; we cover weighted LLE for deformed distributed data, weighted LLE using probability of occurrence, SLLE by adjusting weights, modified LLE, and iterative LLE. 
\end{abstract}

\section{Introduction}

\begin{figure*}[!t]
\centering
\includegraphics[width=6.5in]{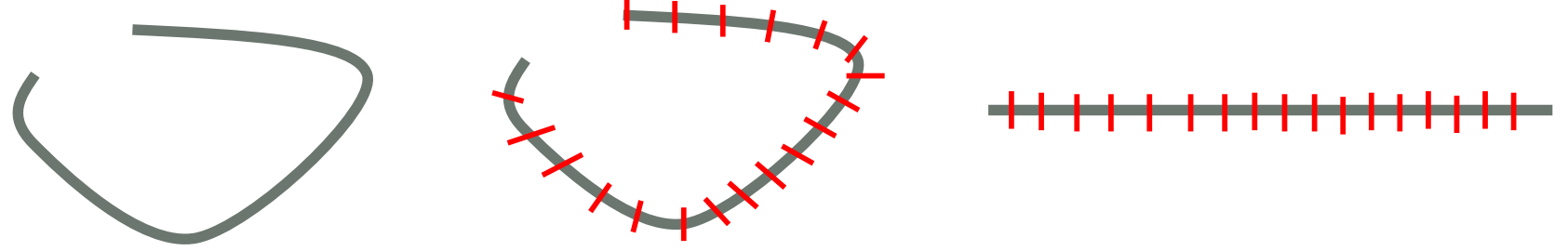}
\caption{Piece-wise local unfolding of manifold by LLE (in this example from two dimensions to one intrinsic dimension). This local unfolding is expected to totally unfold the manifold properly.}
\label{figure_LLE_spline}
\end{figure*}

Locally Linear Embedding (LLE) \cite{roweis2000nonlinear,chen2011locally} is a nonlinear spectral dimensionality reduction method \cite{saul2006spectral} which can be used for manifold embedding and feature extraction \cite{ghojogh2019feature}. LLE tries to preserve the local structure of data in the embedding space. In other words, the close points in the high-dimensional input space should also be close to each other in the low-dimensional embedding space. 
By this local fitting, hopefully the far points in the input space also fall far away from each other in the embedding space. This idea of fitting locally and thinking globally is the main idea of LLE \cite{saul2002think,saul2003think,yotov2005think,wu2018think}. In another perspective, the idea of local fitting by LLE is similar to idea of piece-wise spline regression \cite{marsh2001spline}. LLE unfolds the nonlinear manifold by locally unfolding of manifold piece by piece and it hopes that these local unfoldings result in a suitable total manifold unfolding (see Fig. \ref{figure_LLE_spline}). 
In general, we can say that most of the unsupervised manifold learning methods have the idea of local fitting. On the other hand, most of the supervised manifold learning methods are based on increasing and decreasing the inter- and intra-class variances, respectively \cite{ghojogh2019fisher}. 
We denote the $n$ data points in the input and feature spaces by $\{\b{x}_i \in \mathbb{R}^d\}_{i=1}^n$ and $\{\b{y}_i \in \mathbb{R}^p\}_{i=1}^n$, respectively, where we usually have $p \ll d$.  
LLE has many different applications, such as in medical areas \cite{liu2013locally,he2020discriminative}. 

The remainder of this paper is as follows. We explain LLE and kernel LLE in Sections \ref{section_LLE} and \ref{section_kernel_LLE}, respectively. Different out-of-sample extensions for LLE are introduced in Section \ref{section_LLE_outOfSample}. 
Section \ref{section_incremental_LLE} explains incremental LLE for streaming data. 
Landmark LLE for big data embedding is explained in Section \ref{section_landmark_LLE}. 
Methods for optimal parameter selection for the number of neighbors are introduced in Section \ref{section_parameter_selection_LLE}. 
Some supervised and semi-supervised LLE are covered in Section \ref{section_supervised_LLE}. 
Robust LLE for handling noise and outliers in LLE are explained in Section \ref{section_robust_LLE}.
We introduce fusion of LLE with other manifold learning methods in Section \ref{section_LLE_fusion_with_others}. Section \ref{section_weighted_LLE} explains weighted LLE. Finally, Section \ref{section_conclusion} concludes the paper. 

\section*{Required Background for the Reader}

This paper assumes that the reader has general knowledge of calculus, linear algebra, and basics of optimization. 

\section{Locally Linear Embedding}\label{section_LLE}

LLE, first proposed in \cite{roweis2000nonlinear} and developed in \cite{saul2000introduction,saul2003think}, has three steps \cite{ghodsi2006dimensionality}. First, it finds $k$-Nearest Neighbors ($k$NN) graph of all training points. Then, it tries to find weights for reconstructing every point by its neighbors, using linear combination. Using the same found weights, it embeds every point by a linear combination of its embedded neighbors. 
The main idea of LLE is using the same reconstruction weights in the lower dimensional embedding space as in the high dimensional input space. 
Figure \ref{figure_LLE} illustrates these three steps.

\begin{figure*}[!t]
\centering
\includegraphics[width=6.5in]{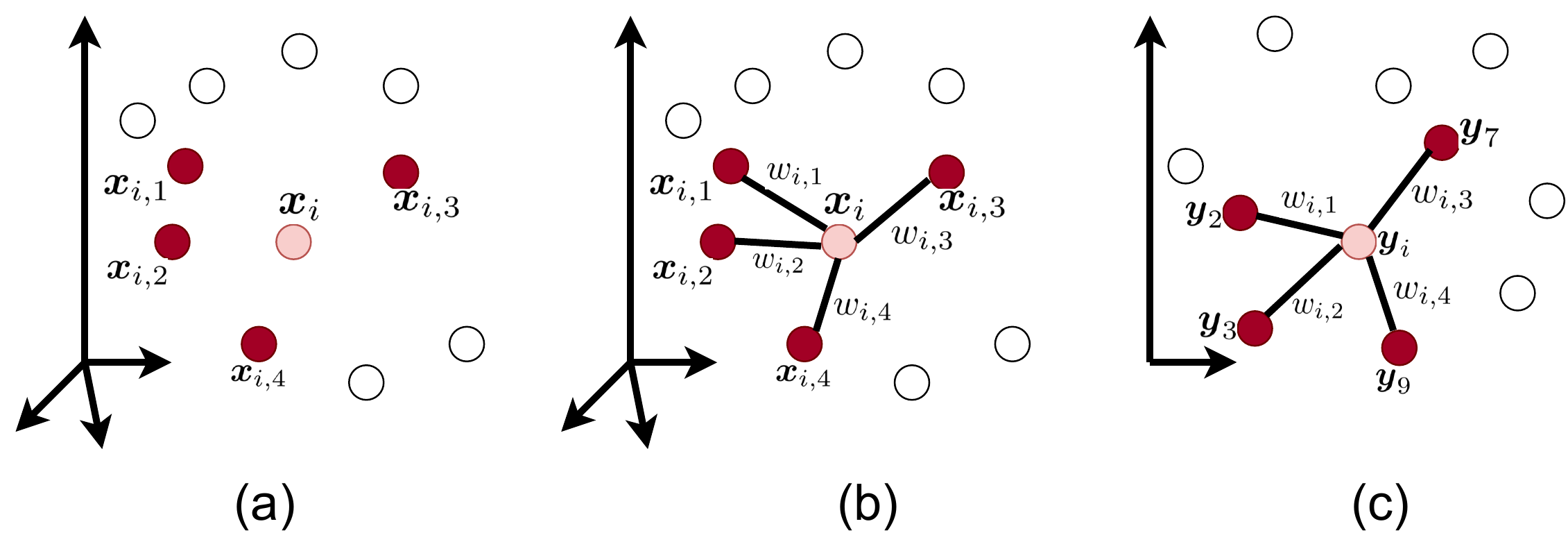}
\caption{Steps in LLE for embedding high dimensional data in a lower dimensional embedding space: (a) finding $k$-nearest neighbors, (b) linear reconstruction by the neighbors, and (c) linear embedding using the calculated weights. In this figure, it is assumed that $k=4$, $\b{x}_{i,1} = \b{x}_2$, $\b{x}_{i,2} = \b{x}_3$, $\b{x}_{i,3} = \b{x}_7$, and $\b{x}_{i,4} = \b{x}_9$.}
\label{figure_LLE}
\end{figure*}

\subsection{$k$-Nearest Neighbors}

A $k$NN graph is formed using pairwise Euclidean distance between the data points. Therefore, every data point has $k$ neighbors. Let $\b{x}_{ij} \in \mathbb{R}^d$ denote the $j$-th neighbor of $\b{x}_i$ and let the matrix $\mathbb{R}^{d \times k} \ni \b{X}_i := [\b{x}_{i1}, \dots, \b{x}_{ik}]$ include the $k$ neighbors of $\b{x}_i$.

\subsection{Linear Reconstruction by the Neighbors}

In the second step, we find the weights for linear reconstruction of every point by its $k$NN. The optimization for this linear reconstruction in the high dimensional input space is formulated as:
\begin{equation}\label{equation_LLE_linearReconstruct}
\begin{aligned}
& \underset{\widetilde{\b{W}}}{\text{minimize}}
& & \varepsilon(\widetilde{\b{W}}) := \sum_{i=1}^n \Big|\Big|\b{x}_i - \sum_{j=1}^k \widetilde{w}_{ij} \b{x}_{ij}\Big|\Big|_2^2, \\
& \text{subject to}
& & \sum_{j=1}^k \widetilde{w}_{ij} = 1, ~~~ \forall i \in \{1, \dots, n\},
\end{aligned}
\end{equation}
where $\mathbb{R}^{n \times k} \ni \widetilde{\b{W}} := [\widetilde{\b{w}}_1, \dots, \widetilde{\b{w}}_n]^\top$ includes the weights, $\mathbb{R}^k \ni \widetilde{\b{w}}_i := [\widetilde{w}_{i1}, \dots, \widetilde{w}_{ik}]^\top$ includes the weights of linear reconstruction of the $i$-th data point using its $k$ neighbors, and $\b{x}_{ij} \in \mathbb{R}^d$ is the $j$-th neighbor of the $i$-th data point.

The constraint $\sum_{j=1}^k \widetilde{w}_{ij} = 1$ means that the weights of linear reconstruction sum to one for every point. Note that the fact that some weights may be negative causes the problem of explosion of some weights because very large positive and negative weights can cancel each other to have a total sum of one. However, this problem does not occur because, as we will see, the solution to this optimization problem has a closed form; thus, weights do not explode. If the solution was found iteratively, the weights would grow and explode gradually \cite{ghojogh2019locally}.  

We can restate the objective $\varepsilon(\widetilde{\b{W}})$ as:
\begin{align}
\varepsilon(\widetilde{\b{W}}) = \sum_{i=1}^n ||\b{x}_i - \b{X}_i \widetilde{\b{w}}_i||_2^2.
\end{align}
The constraint $\sum_{j=1}^k \widetilde{w}_{ij} = 1$ implies that $\b{1}^\top \widetilde{\b{w}}_i = 1$; therefore, $\b{x}_i = \b{x}_i \b{1}^\top \widetilde{\b{w}}_i$. We can simplify the term in $\varepsilon(\widetilde{\b{W}})$ as:
\begin{align*}
||\b{x}_i &- \b{X}_i \widetilde{\b{w}}_i||_2^2 = ||\b{x}_i \b{1}^\top \widetilde{\b{w}}_i - \b{X}_i \widetilde{\b{w}}_i||_2^2 \\
&= ||(\b{x}_i \b{1}^\top - \b{X}_i)\, \widetilde{\b{w}}_i||_2^2 \\
&= \widetilde{\b{w}}_i^\top (\b{x}_i \b{1}^\top - \b{X}_i)^\top (\b{x}_i \b{1}^\top - \b{X}_i)\, \widetilde{\b{w}}_i \\
&= \widetilde{\b{w}}_i^\top \b{G}_i\, \widetilde{\b{w}}_i,
\end{align*}
where $\b{G}_i$ is a gram matrix defined as:
\begin{align}\label{equation_G}
\mathbb{R}^{k \times k} \ni \b{G}_i := (\b{x}_i \b{1}^\top - \b{X}_i)^\top (\b{x}_i \b{1}^\top - \b{X}_i).
\end{align}
Finally, Eq. (\ref{equation_LLE_linearReconstruct}) can be rewritten as:
\begin{equation}\label{equation_LLE_linearReconstruct_2}
\begin{aligned}
& \underset{\{\widetilde{\b{w}}_i\}_{i=1}^n}{\text{minimize}}
& & \sum_{i=1}^n \widetilde{\b{w}}_i^\top \b{G}_i\, \widetilde{\b{w}}_i, \\
& \text{subject to}
& & \b{1}^\top \widetilde{\b{w}}_i = 1, ~~~ \forall i \in \{1, \dots, n\}.
\end{aligned}
\end{equation}
The Lagrangian for Eq. (\ref{equation_LLE_linearReconstruct_2}) is \cite{boyd2004convex}:
\begin{align*}
\mathcal{L} = \sum_{i=1}^n \widetilde{\b{w}}_i^\top \b{G}_i\, \widetilde{\b{w}}_i - \sum_{i=1}^n \lambda_i\, (\b{1}^\top \widetilde{\b{w}}_i - 1).
\end{align*}
Setting the derivative of Lagrangian to zero gives:
\begin{align}
\mathbb{R}^{k} \ni \frac{\partial \mathcal{L}}{\partial \widetilde{\b{w}}_i} &= 2 \b{G}_i \widetilde{\b{w}}_i - \lambda_i \b{1} \overset{\text{set}}{=} \b{0}, \nonumber \\
&\implies \widetilde{\b{w}}_i = \frac{1}{2} \b{G}_i^{-1} \lambda_i \b{1} = \frac{\lambda_i}{2} \b{G}_i^{-1} \b{1}. \label{equation_derivative_lagrangian_1} \\
\mathbb{R} \ni \frac{\partial \mathcal{L}}{\partial \lambda} &= \b{1}^\top \widetilde{\b{w}}_i - 1 \overset{\text{set}}{=} 0 \implies \b{1}^\top \widetilde{\b{w}}_i = 1. \label{equation_derivative_lagrangian_2}
\end{align}
Using Eqs. (\ref{equation_derivative_lagrangian_1}) and (\ref{equation_derivative_lagrangian_2}), we have:
\begin{align}
\frac{\lambda_i}{2} \b{1}^\top \b{G}_i^{-1} \b{1} = 1 \implies \lambda_i = \frac{2}{\b{1}^\top \b{G}_i^{-1} \b{1}}. \label{equation_derivative_lagrangian_3}
\end{align}
Using Eqs. (\ref{equation_derivative_lagrangian_1}) and (\ref{equation_derivative_lagrangian_3}), we have:
\begin{align}\label{equation_w_tilde_solution}
\widetilde{\b{w}}_i = \frac{\lambda_i}{2} \b{G}_i^{-1} \b{1} = \frac{\b{G}_i^{-1} \b{1}}{\b{1}^\top \b{G}_i^{-1} \b{1}}.
\end{align}
According to Eq. (\ref{equation_G}), the rank of matrix $\b{G}_i \in \mathbb{R}^{k \times k}$ is at most equal to $\min(k,d)$. If $d < k$, then $\b{G}_i$ is singular and $\b{G}_i$ should be replaced by $\b{G}_i + \epsilon \b{I}$ where $\epsilon$ is a small positive number. Usually, the data are high dimensional (so $k \ll d$) like in images and thus if $\b{G}_i$ is full rank, we will not have any problem with inverting it. This strengthening the main diagonal of $\b{G}$ is referred to as regularization in LLE \cite{daza2010regularization}. This numerical technique is widely used in manifold and subspace learning (e.g., see \cite{mika1999fisher}).

\subsection{Linear Embedding}\label{section_LLE_linearEmbedding}

In the second step, we found the weights for linear reconstruction in the high dimensional input space. In the third step, we embed data in the low dimensional embedding space using the same weights as in the input space. This linear embedding can be formulated as the following optimization problem: 
\begin{equation}\label{equation_LLE_linearEmbedding}
\begin{aligned}
& \underset{\b{Y}}{\text{minimize}}
& & \sum_{i=1}^n \Big|\Big|\b{y}_i - \sum_{j=1}^n w_{ij} \b{y}_j\Big|\Big|_2^2, \\
& \text{subject to}
& & \frac{1}{n} \sum_{i=1}^n \b{y}_i \b{y}_i^\top = \b{I}, \\
& & & \sum_{i=1}^n \b{y}_i = \b{0},
\end{aligned}
\end{equation}
where $\b{I}$ is the identity matrix, the rows of $\mathbb{R}^{n \times p} \ni \b{Y} := [\b{y}_1, \dots, \b{y}_n]^\top$ are the embedded data points (stacked row-wise), $\b{y}_i \in \mathbb{R}^p$ is the $i$-th embedded data point, and $w_{ij}$ is the weight obtained from the linear reconstruction if $\b{x}_j$ is a neighbor of $\b{x}_i$ and zero otherwise:
\begin{align}\label{equaion_LLE_weight_and_weightHat}
w_{ij} := 
\left\{
    \begin{array}{ll}
        \widetilde{w}_{ij} & \mbox{if } \b{x}_j \in k\text{NN}(\b{x}_i) \\
        0 & \mbox{otherwise}.
    \end{array}
\right.
\end{align}

The second constraint in Eq. (\ref{equation_LLE_linearEmbedding}) ensures the zero mean of embedded data points. The first and second constraints together satisfy having unit covariance for the embedded points.

Suppose $\mathbb{R}^n \ni \b{w}_i := [w_{i1}, \dots, w_{in}]^\top$ and let $\mathbb{R}^n \ni \b{1}_i := [0, \dots, 1, \dots, 0]^\top$ be the vector whose $i$-th element is one and other elements are zero.
The objective function in Eq. (\ref{equation_LLE_linearEmbedding}) can be restated as:
\begin{align*}
\sum_{i=1}^n \Big|\Big|\b{y}_i - \sum_{j=1}^n w_{ij} \b{y}_j\Big|\Big|_2^2 = \sum_{i=1}^n ||\b{Y}^\top\b{1}_i - \b{Y}^\top\b{w}_i||_2^2, 
\end{align*}
which can be stated in matrix form:
\begin{align}\label{equation_LLE_linearEmbedding_objective}
\sum_{i=1}^n ||\b{Y}^\top\b{1}_i &- \b{Y}^\top\b{w}_i||_2^2 = ||\b{Y}^\top\b{I} - \b{Y}^\top\b{W}^\top||_F^2 \nonumber \\
&= ||\b{Y}^\top (\b{I} - \b{W})^\top||_F^2,
\end{align}
where the $i$-th row of $\mathbb{R}^{n \times n} \ni \b{W} := [\b{w}_1, \dots, \b{w}_n]^\top$ includes the weights for the $i$-th data point and $||.||_F$ denotes the Frobenius norm of matrix.
The Eq. (\ref{equation_LLE_linearEmbedding_objective}) is simplified as:
\begin{align}
||\b{Y}^\top(\b{I} - \b{W})^\top||_F^2 &= \textbf{tr}\big((\b{I} - \b{W})\b{Y}\b{Y}^\top(\b{I} - \b{W})^\top\big) \nonumber \\
&= \textbf{tr}\big(\b{Y}^\top(\b{I} - \b{W})^\top(\b{I} - \b{W})\b{Y}\big) \nonumber \\
&= \textbf{tr}(\b{Y}^\top\b{M}\b{Y}),
\end{align}
where $\textbf{tr}(.)$ denotes the trace of matrix and:
\begin{align}\label{equation_M}
\mathbb{R}^{n \times n} \ni \b{M} := (\b{I} - \b{W})^\top (\b{I} - \b{W}).
\end{align}
Note that $(\b{I} - \b{W})$ is the Laplacian of matrix $\b{W}$ because the columns of $\b{W}$, which are $\b{w}_i$'s, add to one (for the constraint used in Eq. (\ref{equation_LLE_linearReconstruct})). Hence, according to Eq. (\ref{equation_M}), the matrix $\b{M}$ can be considered as the gram matrix over the Laplacian of weight matrix.

Finally, Eq. (\ref{equation_LLE_linearEmbedding}) can be rewritten as:
\begin{equation}\label{equation_LLE_linearEmbedding_2}
\begin{aligned}
& \underset{\b{Y}}{\text{minimize}}
& & \textbf{tr}(\b{Y}^\top\b{M}\b{Y}), \\
& \text{subject to}
& & \frac{1}{n} \b{Y}^\top \b{Y} = \b{I}, \\
& & & \b{Y}^\top \b{1} = \b{0},
\end{aligned}
\end{equation}
where the dimensionality of $\b{1}$ and $\b{0}$ are $\mathbb{R}^n$ and $\mathbb{R}^p$, respectively.
Note that we will show in Section \ref{section_interpret_LLE_using_HSIC} that Eq. (\ref{equation_LLE_linearEmbedding_2}) can be interpreted as maximization of dependence between the input data $\b{X}$ and the embedding $\b{Y}$. 
We will show later, in Proposition \ref{proposition_implicite_second_constraint}, that the second constraint will be satisfied implicitly. Therefore, if we ignore the second constraint, the Lagrangian for Eq. (\ref{equation_LLE_linearEmbedding_2}) is \cite{boyd2004convex}:
\begin{align*}
\mathcal{L} = \textbf{tr}(\b{Y}^\top\b{M}\b{Y}) - \textbf{tr}\big(\b{\Lambda}^\top (\frac{1}{n} \b{Y}^\top \b{Y} - \b{I})\big),
\end{align*}
where $\b{\Lambda} \in \mathbb{R}^{n \times n}$ is a diagonal matrix including the Lagrange multipliers. 
Equating derivative of $\mathcal{L}$ to zero gives us:
\begin{align}
&\mathbb{R}^{n \times p} \ni \frac{\partial \mathcal{L}}{\partial \b{Y}} = 2\b{M}\b{Y} - \frac{2}{n} \b{Y} \b{\Lambda} \overset{\text{set}}{=} \b{0} \nonumber \\
&\implies \b{M}\b{Y} = \b{Y} (\frac{1}{n}\b{\Lambda}), \label{equation_LLE_linearEmbedding_eigenproblem}
\end{align}
which is the eigenvalue problem for $\b{M}$ \cite{ghojogh2019eigenvalue}. Therefore, the columns of $\b{Y}$ are the eigenvectors of $\b{M}$ where eigenvalues are the diagonal elements of $(1/n)\b{\Lambda}$.

As Eq. (\ref{equation_LLE_linearEmbedding_2}) is a minimization problem, the columns of $\b{Y}$ should be sorted from the smallest to largest eigenvalues.
Moreover, recall that we explained $(\b{I} - \b{W})$ in $\b{M}$ is the Laplacian matrix for the weights $\b{W}$.
It is well-known in linear algebra and graph theory that if a graph has $k$ disjoint connected parts, its Laplacian matrix has $k$ zero eigenvalues (see {\citep[Theorem 3.10]{marsden2013eigenvalues}} and \cite{polito2002grouping,ahmadizadeh2017eigenvalues}).
As the $k$NN graph, or $\b{W}$, is a connected graph, $(\b{I} - \b{W})$ has one zero eigenvalue whose eigenvector is $\b{1} = [1, 1, \dots, 1]^\top$. After sorting the eigenvectors from smallest to largest eigenvalues, we ignore the first eigenvector having zero eigenvalue and take the $p$ smallest eigenvectors of $\b{M}$ with non-zero eigenvalues as the columns of $\b{Y} \in \mathbb{R}^{n \times p}$.

\begin{proposition}\label{proposition_implicite_second_constraint}
The fact that we have the eigenvector $\b{1}$ with zero eigenvalue implicitly ensures that $\sum_{i=1}^n \b{y}_i = \b{Y}^\top \b{1} = \b{0}$ which was the second constraint.
\end{proposition}
\begin{proof}
Suppose the eigenvectors are sorted from the smallest to largest eigenvalues. Let $\b{v}_i \in \mathbb{R}^n$ and $\lambda_i \in \mathbb{R}$ be the $i$-th eigenvector and eigenvalue, respectively. Therefore, in Eq. (\ref{equation_LLE_linearEmbedding_eigenproblem}), if we consider all the eigenvectors and not just $p$ of them, we have $\b{Y} = [\b{y}_1, \dots, \b{y}_n]^\top = [\b{v}_1, \dots, \b{v}_n] \in \mathbb{R}^{n \times n}$.
We know that eigenvectors are orthogonal by definition; therefore, $\b{v}_1^\top \b{v}_i = 0, \forall i \neq 1$.
We know that $\b{v}_1 = \b{1}$ with $\lambda_1 = 0$; therefore, $\b{1}^\top \b{v}_i = 0$ which means that the elements of every eigenvector, $\b{v}_i, \forall i \neq 1$, add to zero. On the other hand, we have $[\b{y}_1, \dots, \b{y}_n]^\top = [\b{v}_1, \dots, \b{v}_n]$ so the summation of a component amongst the $\b{y}_i$'s (embedded data points) is zero. As the summation for `every' component amongst $\b{y}_i$'s is zero, we have $\sum_{i=1}^n \b{y}_i = \b{0}$.
This explanation can be summarized in this sentence: ``Discarding this eigenvector enforces the constraint that the outputs have zero mean, since the components of other eigenvectors must sum to zero, by virtue of orthogonality with the bottom one (with smallest eigenvalue)'' \cite{saul2003think}. Q.E.D.
\end{proof}

\subsection{Additional Notes on LLE}

\subsubsection{Inverse Locally Linear Embedding}


We can have inverse LLE where we find the data point $\b{x}_i \in \mathbb{R}^d$ in the input space for an embedding point $\b{y}_i \in \mathbb{R}^p$ {\citep[Section 6.1]{saul2003think}}. 
For this, we find $k$NN in the embedding space; let $\b{y}_{ij}$ denote the $j$-th neighbor of $\b{y}_i$ in the embedding space.
We solve the following problem to find the reconstruction weights, $\{\widetilde{w}_{ij}\}_{j=1}^k$, in the embedding space:
\begin{equation}
\begin{aligned}
& \underset{\{\widetilde{w}_{ij}\}_{j=1}^k}{\text{minimize}}
& & \Big|\Big|\b{y}_i - \sum_{j=1}^k \widetilde{w}_{ij}\, \b{y}_{ij}\Big|\Big|_2^2, \\
& \text{subject to}
& & \sum_{j=1}^k \widetilde{w}_{ij} = 1,
\end{aligned}
\end{equation}
which is solved similar to how Eq. (\ref{equation_LLE_linearReconstruct}) is solved. 
Thereafter, $\{w_{ij}\}_{j=1}^k$ is obtained by the obtained $\{\widetilde{w}_{ij}\}_{j=1}^k$ using Eq. (\ref{equaion_LLE_weight_and_weightHat}).  
The original point in the input space is approximated using the obtained reconstruction weights:
\begin{align}
\mathbb{R}^d \ni \b{x}_i \approx \sum_{j=1}^k w_{ij}\, \b{x}_{j}.
\end{align}

\subsubsection{Feature Fusion in LLE}

It is noteworthy that, in some cases, data points are represented by $Q$ different features, i.e., we have $\{\b{x}_i^q | i = 1,\dots,n, q=1,\dots,Q\}$. In these cases, we need feature fusion using LLE \cite{sun2009feature} where $Q$ weights are obtained by LLE, denoted by $\b{W}_1, \dots, \b{W}_Q$. The weights can be combined as:
\begin{align}
\bar{\b{W}} := \frac{1}{Q} \sum_{q=1}^Q \b{W}_q,
\end{align}
and $\bar{\b{W}}$ is used in Eq. (\ref{equation_M}) rather than $\b{W}$. The embedding optimization is used for finding the embeddings, although the data have several features \cite{sun2009feature}.

\section{Kernel Locally Linear Embedding}\label{section_kernel_LLE}

We can map data $\{\b{x}_i \in \mathbb{R}^d\}_{i=1}^n$ to higher-dimensional feature space hoping to have data fall close to a simpler-to-analyze manifold in the feature space. Suppose $\b{\phi}: \b{x} \rightarrow \mathcal{H}$ is the pulling function which maps data $\{\b{x}_i\}_{i=1}^n$ to the feature space. In other words, $\b{x}_i \mapsto \b{\phi}(\b{x}_i)$. Let $t$ denote the dimensionality of the feature space, i.e., $\b{\phi}(\b{x}_i) \in \mathbb{R}^t$. We usually have $t \gg d$. The kernel of two data points $\b{x}_1$ and $\b{x}_2$ is $\b{\phi}(\b{x}_1)^\top \b{\phi}(\b{x}_2) \in \mathbb{R}$ \cite{hofmann2008kernel}. The kernel matrix for the $n$ data points is $\mathbb{R}^{n \times n} \ni \b{K} := \b{\Phi}(\b{X})^\top \b{\Phi}(\b{X})$ where $\b{\Phi}(\b{X}) := [\b{\phi}(\b{x}_{1}), \dots, \b{\phi}(\b{x}_{n})] \in \mathbb{R}^{t \times n}$.
Kernel LLE \cite{zhao2012facial} maps data to the feature space and performs the steps of $k$NN and linear reconstruction in the feature space.

\subsection{$k$-Nearest Neighbors}

The Euclidean distance in the feature space is \cite{scholkopf2001kernel}:
\begin{align}
&||\b{\phi}(\b{x}_i) - \b{\phi}(\b{x}_j)||_2 \nonumber \\
&= \sqrt{\big(\b{\phi}(\b{x}_i) - \b{\phi}(\b{x}_j)\big)^\top\big(\b{\phi}(\b{x}_i) - \b{\phi}(\b{x}_j)\big)} \nonumber \\
&= \sqrt{\b{\phi}(\b{x}_i)^\top\b{\phi}(\b{x}_i) -2 \b{\phi}(\b{x}_i)^\top\b{\phi}(\b{x}_j) + \b{\phi}(\b{x}_j)^\top\b{\phi}(\b{x}_j)} \nonumber \\
&= \sqrt{k(\b{x}_i, \b{x}_i) -2 k(\b{x}_i, \b{x}_j) + k(\b{x}_j, \b{x}_j)}, \label{equation_distance_featureSpace}
\end{align}
where $\mathbb{R} \ni k(\b{x}_i, \b{x}_j) = \b{\phi}(\b{x}_i)^\top\b{\phi}(\b{x}_j)$ is the $(i,j)$-th element of $\b{K}$.

Using the distances of the data points in the feature space, i.e. Eq. (\ref{equation_distance_featureSpace}), we construct the $k$NN graph. 
Therefore, every data point has $k$ neighbors in the feature space. Let matrix $\mathbb{R}^{t \times k} \ni \b{\Phi}(\b{X}_i) := [\b{\phi}(\b{x}_{i1}), \dots, \b{\phi}(\b{x}_{ik})]$ include the neighbors of $\b{x}_i$ in the feature space.

\subsection{Linear Reconstruction by the Neighbors}

The Eq. (\ref{equation_LLE_linearReconstruct}) in the feature space is:
\begin{equation}\label{equation_kernel_LLE_linearReconstruct}
\begin{aligned}
& \underset{\widetilde{\b{W}}}{\text{minimize}}
& & \varepsilon(\widetilde{\b{W}}) := \sum_{i=1}^n \Big|\Big|\b{\phi}(\b{x}_i) - \sum_{j=1}^k \widetilde{w}_{ij} \b{\phi}(\b{x}_{ij})\Big|\Big|_2^2, \\
& \text{subject to}
& & \sum_{j=1}^k \widetilde{w}_{ij} = 1, ~~~ \forall i \in \{1, \dots, n\}.
\end{aligned}
\end{equation}
We can restate $\varepsilon(\widetilde{\b{W}})$ as:
\begin{align*}
\varepsilon(\widetilde{\b{W}}) &= \sum_{i=1}^n \Big|\Big|\b{\phi}(\b{x}_i) - \sum_{j=1}^k \widetilde{w}_{ij} \b{\phi}(\b{x}_{ij})\Big|\Big|_2^2 \\
&\overset{(a)}{=} \sum_{i=1}^n \Big|\Big|\sum_{j=1}^k \widetilde{w}_{ij} \b{\phi}(\b{x}_i) - \sum_{j=1}^k \widetilde{w}_{ij} \b{\phi}(\b{x}_{ij})\Big|\Big|_2^2 \\
&= \sum_{i=1}^n \Big|\Big|\sum_{j=1}^k \widetilde{w}_{ij} \big(\b{\phi}(\b{x}_i) - \b{\phi}(\b{x}_{ij})\big)\Big|\Big|_2^2,
\end{align*}
where $(a)$ is because $\sum_{j=1}^k \widetilde{w}_{ij} = 1$.
We define:
\begin{equation}
\begin{aligned}
\mathbb{R}^{t \times k} &\ni \b{P}_i = [\b{p}_{i1}, \dots, \b{p}_{ik}] \\
&:= \big[\b{\phi}(\b{x}_i) - \b{\phi}(\b{x}_{i1}), \dots, \b{\phi}(\b{x}_i) - \b{\phi}(\b{x}_{ik})\big].
\end{aligned}
\end{equation}
Therefore:
\begin{align}
\varepsilon(\widetilde{\b{W}}) &= \sum_{i=1}^n \Big|\Big|\sum_{j=1}^k \widetilde{w}_{ij} \big(\b{\phi}(\b{x}_i) - \b{\phi}(\b{x}_{ij})\big)\Big|\Big|_2^2 \nonumber \\
&= \sum_{i=1}^n \Big|\Big|\sum_{j=1}^k \widetilde{w}_{ij} \b{p}_{ij}\Big|\Big|_2^2 = \sum_{i=1}^n ||\b{P}_i\widetilde{\b{w}}_i||_2^2 \nonumber \\
&= \sum_{i=1}^n (\b{P}_i\widetilde{\b{w}}_i)^\top(\b{P}_i\widetilde{\b{w}}_i) = \sum_{i=1}^n \widetilde{\b{w}}_i^\top \b{P}_i^\top \b{P}_i\widetilde{\b{w}}_i \nonumber \\
&= \sum_{i=1}^n \widetilde{\b{w}}_i^\top \b{K}_i \widetilde{\b{w}}_i,
\end{align}
where:
$\mathbb{R}^{k \times k} \ni \b{K}_i := \b{P}_i^\top \b{P}_i$. The $(a,b)$-th element of $\b{K}_i$ can be calculated as:
\begin{align}
&\b{K}_i(a,b) = \b{p}_{ia}^\top\, \b{p}_{ib} \nonumber \\
&= \big(\b{\phi}(\b{x}_i) - \b{\phi}(\b{x}_{ia})\big)^\top \big(\b{\phi}(\b{x}_i) - \b{\phi}(\b{x}_{ib})\big) \nonumber \\
&= \b{\phi}(\b{x}_i)^\top \b{\phi}(\b{x}_i) - \b{\phi}(\b{x}_i)^\top \b{\phi}(\b{x}_{ia}) \nonumber \\
&- \b{\phi}(\b{x}_i)^\top \b{\phi}(\b{x}_{ib}) + \b{\phi}(\b{x}_{ia})^\top \b{\phi}(\b{x}_{ib}) \nonumber \\
&= k(\b{x}_i, \b{x}_i) - k(\b{x}_i, \b{x}_{ia}) - k(\b{x}_i, \b{x}_{ib}) + k(\b{x}_{ia}, \b{x}_{ib}). \label{equation_kernel_pTranspose_p}
\end{align}
Therefore the Eq. (\ref{equation_kernel_LLE_linearReconstruct}) is restated to:
\begin{equation}\label{equation_kernel_LLE_linearReconstruct_2}
\begin{aligned}
& \underset{\{\widetilde{\b{w}}_i\}_{i=1}^n}{\text{minimize}}
& & \sum_{i=1}^n \widetilde{\b{w}}_i^\top \b{K}_i\, \widetilde{\b{w}}_i, \\
& \text{subject to}
& & \b{1}^\top \widetilde{\b{w}}_i = 1, ~~~ \forall i \in \{1, \dots, n\}.
\end{aligned}
\end{equation}
The Lagrangian for Eq. (\ref{equation_kernel_LLE_linearReconstruct_2}) is \cite{boyd2004convex}:
\begin{align*}
\mathcal{L} = \widetilde{\b{w}}_i^\top \b{K}_i\, \widetilde{\b{w}}_i - \sum_{i=1}^n \lambda_i\, (\b{1}^\top \widetilde{\b{w}}_i - 1).
\end{align*}
Setting the derivative of Lagrangian to zero gives:
\begin{align}
\mathbb{R}^{k} \ni \frac{\partial \mathcal{L}}{\partial \widetilde{\b{w}}_i} &= 2 \b{K}_i \widetilde{\b{w}}_i - \lambda_i \b{1} \overset{\text{set}}{=} \b{0}, \nonumber \\
&\implies \widetilde{\b{w}}_i = \frac{1}{2} \b{K}_i^{-1} \lambda_i \b{1} = \frac{\lambda_i}{2} \b{K}_i^{-1} \b{1}. \label{equation_kernel_LLE_derivative_lagrangian_1} \\
\mathbb{R} \ni \frac{\partial \mathcal{L}}{\partial \lambda} &= \b{1}^\top \widetilde{\b{w}}_i - 1 \overset{\text{set}}{=} 0 \implies \b{1}^\top \widetilde{\b{w}}_i = 1. \label{equation_kernel_LLE_derivative_lagrangian_2}
\end{align}
Using Eqs. (\ref{equation_kernel_LLE_derivative_lagrangian_1}) and (\ref{equation_kernel_LLE_derivative_lagrangian_2}), we have:
\begin{align}
\frac{\lambda_i}{2} \b{1}^\top \b{K}_i^{-1} \b{1} = 1 \implies \lambda_i = \frac{2}{\b{1}^\top \b{K}_i^{-1} \b{1}}. \label{equation_kernel_LLE_derivative_lagrangian_3}
\end{align}
Using Eqs. (\ref{equation_kernel_LLE_derivative_lagrangian_1}) and (\ref{equation_kernel_LLE_derivative_lagrangian_3}), we have:
\begin{align}
\widetilde{\b{w}}_i = \frac{\lambda_i}{2} \b{K}_i^{-1} \b{1} = \frac{\b{K}_i^{-1} \b{1}}{\b{1}^\top \b{K}_i^{-1} \b{1}}.
\end{align}

\subsection{Linear Embedding}

The linear embedding step in kernel LLE is exactly as the linear embedding step in LLE (see Section \ref{section_LLE_linearEmbedding}). 



\section{Out-of-sample Embedding in LLE}\label{section_LLE_outOfSample}

Suppose we have $n_t$ out-of-sample (test) data points, i.e., $\mathbb{R}^{d \times n_t} \ni \b{X}^{(t)} := [\b{x}_1^{(t)}, \dots, \b{x}_{n_t}^{(t)}]$, which are not used and have not been seen in training. Let $\b{x}_i^{(t)} \in \mathbb{R}^d$ denote the $i$-th out-of-sample data point. We desire to find the low-dimensional embedding of out-of-sample points, denoted by $\{\b{y}_i^{(t)} \in \mathbb{R}^p\}_{i=1}^{n_t}$ or $\b{Y}_t = [\b{y}_1^{(t)}, \dots, \b{y}_n^{(t)}] \in \mathbb{R}^{p \times n_t}$, after the training phase. 
There exist several approaches for out-of-sample extension of LLE. In the following, we explain these methods. 
In addition to the methods introduced in this section, there exist some other methods for out-of-sample extension of LLE such as \cite{bunte2012general}, which we pass by in this paper. 
Moreover, the Incremental LLE \cite{kouropteva2005incremental} and SLLEP \cite{li2011supervised} methods, which can be used for out-of-sample extension of LLE, are not explained in this section because they will be introduced in Sections \ref{section_incremental_LLE} and \ref{section_LLE_projection}, respectively.

\subsection{Out-of-sample Embedding Using Linear Reconstruction}

One way of extending LLE for out-of-sample data point is using linear reconstruction \cite{saul2003think}.

For every out-of-sample data point $\b{x}_i^{(t)}$, we first find the $k$NN among the training points. 
Let $\b{x}_{ij}^{(t)}$ denote the $j$-th training neighbor of $\b{x}_i^{(t)}$ and let matrix $\mathbb{R}^{d \times k} \ni \b{X}_i^{(t)} := [\b{x}_{i1}^{(t)}, \dots, \b{x}_{ik}^{(t)}]$ include the training neighbors of $\b{x}_i^{(t)}$.
We want to reconstruct every out-of-sample point by its training neighbors (not that out-of-sample points are not considered as neighbors). Hence, using an optimization problem similar to Eq. (\ref{equation_LLE_linearReconstruct}), we have:
\begin{equation}\label{equation_outOfSample_LLE_linearReconstruction}
\begin{aligned}
& \underset{\widetilde{\b{W}}^{(t)}}{\text{minimize}}
& & \varepsilon(\widetilde{\b{W}}^{(t)}) := \sum_{i=1}^{n_t} \Big|\Big|\b{x}_i^{(t)} - \sum_{j=1}^k \widetilde{w}_{ij}^{(t)} \b{x}_{ij}^{(t)}\Big|\Big|_2^2, \\
& \text{subject to}
& & \sum_{j=1}^k \widetilde{w}_{ij}^{(t)} = 1, ~~~ \forall i \in \{1, \dots, n_t\},
\end{aligned}
\end{equation}
where $\mathbb{R}^{n_t \times k} \ni \widetilde{\b{W}}^{(t)} := [\widetilde{\b{w}}_1^{(t)}, \dots, \widetilde{\b{w}}_{n_t}^{(t)}]^\top$ includes the weights and $\mathbb{R}^k \ni \widetilde{\b{w}}_i^{(t)} := [\widetilde{w}_{i1}^{(t)}, \dots, \widetilde{w}_{ik}^{(t)}]^\top$ includes the weights of linear reconstruction of the $m$-th out-of-sample data point using its $k$ training neighbors.
We can restate the $\varepsilon(\widetilde{\b{W}}^{(t)})$ as:
\begin{align}
\varepsilon(\widetilde{\b{W}}^{(t)}) = \sum_{i=1}^{n_t} ||\b{x}_i^{(t)} - \b{X}_i^{(t)} \widetilde{\b{w}}_i^{(t)}||_2^2.
\end{align}
The constraint $\sum_{j=1}^k \widetilde{w}_{ij}^{(t)} = 1$ is restated as $\b{1}^\top \widetilde{\b{w}}_i^{(t)} = 1$; therefore, we can say $\b{x}_i^{(t)} = \b{x}_i^{(t)} \b{1}^\top \widetilde{\b{w}}_i$. We can simplify the term in $\varepsilon(\widetilde{\b{W}}^{(t)})$ as:
\begin{align*}
||\b{x}_i^{(t)} &- \b{X}_i^{(t)} \widetilde{\b{w}}_i^{(t)}||_2^2 = ||\b{x}_i^{(t)} \b{1}^\top \widetilde{\b{w}}_i^{(t)} - \b{X}_i^{(t)} \widetilde{\b{w}}_i^{(t)}||_2^2 \\
&= ||(\b{x}_i^{(t)} \b{1}^\top - \b{X}_i^{(t)})\, \widetilde{\b{w}}_i^{(t)}||_2^2 \\
&= \widetilde{\b{w}}_i^{(t)\top} (\b{x}_i^{(t)} \b{1}^\top - \b{X}_i^{(t)})^\top (\b{x}_i^{(t)} \b{1}^\top - \b{X}_i^{(t)})\, \widetilde{\b{w}}_i^{(t)} \\
&= \widetilde{\b{w}}_i^{(t)\top} \b{G}_i^{(t)}\, \widetilde{\b{w}}_i^{(t)},
\end{align*}
where: 
\begin{align}\label{equation_G_outOfSample}
\mathbb{R}^{k \times k} \ni \b{G}_i^{(t)} := (\b{x}_i^{(t)} \b{1}^\top - \b{X}_i^{(t)})^\top (\b{x}_i^{(t)} \b{1}^\top - \b{X}_i^{(t)}).
\end{align}
The Eq. (\ref{equation_outOfSample_LLE_linearReconstruction}) can be rewritten as:
\begin{equation}\label{equation_outOfSample_LLE_linearReconstruction_2}
\begin{aligned}
& \underset{\widetilde{\b{W}}^{(t)}}{\text{minimize}}
& & \sum_{i=1}^{n_t} \widetilde{\b{w}}_i^{(t)\top} \b{G}_i^{(t)}\, \widetilde{\b{w}}_i^{(t)}, \\
& \text{subject to}
& & \b{1}^\top \widetilde{\b{w}}_i^{(t)} = 1, ~~~ \forall i \in \{1, \dots, n_t\}.
\end{aligned}
\end{equation}
This problem is solved similar to the solution for Eq. (\ref{equation_LLE_linearReconstruct_2}). Therefore, similar to Eq. (\ref{equation_w_tilde_solution}), we have:
\begin{align}
\widetilde{\b{w}}_i^{(t)} = \frac{(\b{G}_i^{(t)})^{-1} \b{1}}{\b{1}^\top (\b{G}_i^{(t)})^{-1} \b{1}}.
\end{align}
The embedding of the out-of-sample $\b{x}_i^{(t)}$ is obtained by the linear combination (reconstruction) of the embedding of its $k$ training neighbors:
\begin{align}\label{equation_outOfSample_LLE_linear_reconstruction}
\mathbb{R}^p \ni \b{y}_i^{(t)} = \sum_{j=1}^{k} \widetilde{w}_{ij}^{(t)} \b{y}_j.
\end{align}

\subsection{Out-of-sample Embedding Using Eigenfunctions}\label{section_outOfSample_eigenfunctions}

\subsubsection{Eigenfunctions}

Consider a Hilbert space $\mathcal{H}_p$ of functions with the inner product $\langle f,g \rangle = \int f(x) g(x) p(x) dx$ with density function $p(x)$. In this space, we can consider the kernel function $K_p$:
\begin{align}
(K_p f)(x) = \int K(x,y)\, f(y)\, p(y)\, dy,
\end{align}
where the density function can be approximated empirically. 
The \textit{eigenfunction decomposition} is defined to be \cite{bengio2004learning,bengio2004out}:
\begin{align}\label{equation_eigenfunction_decomposition}
(K_p f_r)(x) = \delta'_r f_r(x),
\end{align}
where $f_r(x)$ is the $r$-th \textit{eigenfunction} and $\delta'_r$ is the corresponding eigenvalue.
If we have the eigenvalue decomposition \cite{ghojogh2019eigenvalue} for the kernel matrix $\b{K}$, we have $\b{K} \b{v}_r = \delta_r \b{v}_r$ where $\b{v}_r$ is the $r$-th eigenvector and $\delta_r$ is the corresponding eigenvalue. According to {\citep[Proposition 1]{bengio2004out}}, we have $\delta'_r = (1/n) \delta_r$. 

\subsubsection{Embedding Using Eigenfunctions}

\begin{proposition}
If $v_{ri}$ is the $i$-th element of the $n$-dimensional vector $\b{v}_r$ and $k(\b{x}, \b{x}_i)$ is the kernel between vectors $\b{x}$ and $\b{x}_i$, the eigenfunction for the point $\b{x}$ and the $i$-th training point $\b{x}_i$ are:
\begin{align}
f_r(\b{x}) &= \frac{\sqrt{n}}{\delta_r} \sum_{i=1}^n v_{ri}\, \breve{k}_t(\b{x}_i, \b{x}), \\
f_r(\b{x}_i) &= \sqrt{n}\, v_{ri}, 
\end{align}
respectively, where $\breve{k}_t(\b{x}_i, \b{x})$ is the centered kernel between training set and the out-of-sample point $\b{x}$. 

Let the LLE embedding of the point $\b{x}$ be $\mathbb{R}^p \ni \b{y}(\b{x}) = [y_1(\b{x}), \dots, y_p(\b{x})]^\top$. The $r$-th dimension of this embedding is:
\begin{align}\label{equation_embedding_eigenfunction}
y_r(\b{x}) &= \sqrt{\delta_r}\, \frac{f_r(\b{x})}{\sqrt{n}} = \frac{1}{\sqrt{\delta_r}} \sum_{i=1}^n v_{ri}\, \breve{k}_t(\b{x}_i, \b{x}).
\end{align}
\end{proposition}

\begin{proof}
This proposition is taken from {\citep[Proposition 1]{bengio2004out}}. For proof, refer to {\citep[Proposition 1]{bengio2004learning}}, {\citep[Proposition 1]{bengio2006spectral}}, and {\citep[Proposition 1 and Theorem 1]{bengio2003spectral}}. 
More complete proofs can be found in \cite{bengio2003learning}. 
\end{proof}

If we have a set of $n_t$ out-of-sample data points, $\breve{k}_t(\b{x}_i, \b{x})$ is an element of the centered out-of-sample kernel (see {\citep[Appendix C]{ghojogh2019unsupervised}}):
\begin{align}
\mathbb{R}^{n \times n_t} \ni \breve{\b{K}}_t &= \b{K}_t - \frac{1}{n} \b{1}_{n \times n} \b{K}_t - \frac{1}{n} \b{K} \b{1}_{n \times n_t} \nonumber \\
&~~~~ + \frac{1}{n^2} \b{1}_{n \times n} \b{K} \b{1}_{n \times n_t}, \label{equation_centered_outOfSample_kernel}
\end{align}
where $\b{1} := [1, 1, \dots, 1]^\top$, $\b{K}_t \in \mathbb{R}^{n \times n_t}$ is the not necessarily centered out-of-sample kernel, and $\b{K} \in \mathbb{R}^{n \times n}$ is the training kernel.

\subsubsection{Out-of-sample Embedding}

One can use Eq. (\ref{equation_embedding_eigenfunction}) to embed the $i$-th out-of-sample data point $\b{x}_i^{(t)}$. For this purpose, $\b{x}_i^{(t)}$ should be used in place of $\b{x}$ in Eq. (\ref{equation_embedding_eigenfunction}). 
Note that Eq. (\ref{equation_embedding_eigenfunction}) requires Eq. (\ref{equation_centered_outOfSample_kernel}).
We require a notion of kernel in LLE. LLE can be seen as a special case of kernel LLE where the inverse or negative sign of $\b{M}$ can be interpreted as its kernel because Eq. (\ref{equation_LLE_linearEmbedding_2}) is a minimization but kernel PCA optimization is a maximization \cite{ghojogh2019unsupervised}. For more details on seeing LLE as kernel PCA, see \cite{scholkopf2002learning,ham2004kernel,bengio2004learning,ghojogh2019feature} and {\citep[Table 2.1]{strange2014open}}.
Hence, the kernel in LLE can be \cite{bengio2004out}:
\begin{align}
&\b{M} \overset{(\ref{equation_M})}{=} (\b{I} - \b{W})^\top (\b{I} - \b{W}) \nonumber \\
&~~~~~~~~~~~~~~~~~~~~~~~ = \b{I} - \b{W} - \b{W}^\top + \b{W}^\top \b{W}, \nonumber \\
&\mathbb{R}^{n \times n} \ni \b{K} := \mu \b{I} - \b{M}, \label{equation_outOfSample_LLE_kernel_K_and_M} \\
&\therefore~~~ \b{K}(i, j) = (\mu - 1)\, \delta_{ij} + w_{ij} + w_{ji} - \sum_{r=1}^n w_{ri}\, w_{rj}, \label{equation_outOfSample_LLE_kernel_K}
\end{align}
where $\delta_{ij}$ is the Kronecker delta which is one if $i=j$ and zero otherwise. 
If we modify the hard similarity $\delta_{ij}$ to a soft similarity $w_{ij}^{(t)}$, Eq. (\ref{equation_outOfSample_LLE_kernel_K}) can be slightly modified to \cite{bengio2004out}:
\begin{align}
&\b{K}(\b{x}_i, \b{x}_j) = (\mu - 1)\, w_{ij}^{(t)} + w_{ij}^{(t)} + w_{ji}^{(t)} - \sum_{r=1}^n w_{ri}^{(t)}\, w_{rj}^{(t)}, \label{equation_outOfSample_LLE_kernel_K_2}
\end{align}
where either $\b{x}_i$ or $\b{x}_j$ (and not both of them) is an out-of-sample data point, i.e., we either have $\b{x}_i^{(t)}$ or $\b{x}_j^{(t)}$. 
We define the first and second terms in Eq. (\ref{equation_outOfSample_LLE_kernel_K_2}) as \cite{bengio2004out}:
\begin{align*}
&\b{K}_t'(\b{x}_i, \b{x}_j) := w_{ij}^{(t)}, \\
&\b{K}_t''(\b{x}_i, \b{x}_j) := w_{ij}^{(t)} + w_{ji}^{(t)} - \sum_{r=1}^n w_{ri}^{(t)}\, w_{rj}^{(t)},
\end{align*}
respectively. Hence, Eq. (\ref{equation_outOfSample_LLE_kernel_K_2}) can be restated as:
\begin{align}\label{equation_outOfSample_LLE_kernel_K_3}
&\b{K}(\b{x}_i, \b{x}_j) = (\mu - 1)\, \b{K}_t'(\b{x}_i, \b{x}_j) + \b{K}_t''(\b{x}_i, \b{x}_j).
\end{align}
In LLE, the embeddings are the eigenvectors of $\b{M}$. According to Eq. (\ref{equation_outOfSample_LLE_kernel_K_and_M}), the embeddings $\b{y}$'s are the eigenvectors of the kernel, previously denoted by $\b{v}$'s.  Hence, we can employ Eq. (\ref{equation_embedding_eigenfunction}) in which the kernel of LLE, Eq. (\ref{equation_outOfSample_LLE_kernel_K_3}), is used \cite{bengio2004out} (we change the dummy iterator $i$ to $j$):
\begin{align}
&y_r(\b{x}) = \nonumber \\
&~~~~~~ \frac{1}{\sqrt{\delta_r}} \sum_{j=1}^n y_{jr}\, \Big( (\mu - 1)\, \b{K}_t'(\b{x}_j, \b{x}) + \b{K}_t''(\b{x}_j, \b{x}) \Big).
\end{align}
Hence, the $r$-th element of out-of-sample embedding $\mathbb{R}^p \ni \b{y}_i^{(t)} = [y_{i1}, \dots, y_{ip}]^\top$ for $\b{x}_i^{(t)} \in \mathbb{R}^d$ is:
\begin{align}\label{equation_outOfSample_LLE_eigenfunctions}
&y_{ir}^{(t)} = \frac{1}{\sqrt{\delta_r}} \sum_{j=1}^n y_{jr}\, \Big( (\mu - 1)\, w_{ij}^{(t)} + \b{K}_t''(\b{x}_j, \b{x}_i^{(t)}) \Big),
\end{align}
where $y_{jr}$ denotes the $r$-th element of $\b{y}_j \in \mathbb{R}^p$. 

\begin{corollary}
The out-of-sample embedding by linear reconstruction, i.e. Eq. (\ref{equation_outOfSample_LLE_linear_reconstruction}), is a special case of out-of-sample embedding by eigenfunctions, i.e. Eq. (\ref{equation_outOfSample_LLE_eigenfunctions}), for $\mu \rightarrow \infty$. 
\end{corollary}
\begin{proof}
On one hand, inspired by Eq. (\ref{equaion_LLE_weight_and_weightHat}), we can restate Eq. (\ref{equation_outOfSample_LLE_linear_reconstruction}) as:
\begin{align*}
\mathbb{R}^p \ni \b{y}_i^{(t)} = \sum_{j=1}^{k} \widetilde{w}_{ij}^{(t)} \b{y}_j = \sum_{j=1}^{n} w_{ij}^{(t)} \b{y}_j,
\end{align*}
whose element-wise expression is:
\begin{align*}
&y_{ir}^{(t)} = \sum_{j=1}^n w_{ij}^{(t)}\, y_{jr}.
\end{align*}
On the other hand, by $\mu \rightarrow \infty$, the first term in Eq. (\ref{equation_outOfSample_LLE_eigenfunctions}) dominates its second term as:
\begin{align*}
&y_{ir}^{(t)} = \frac{\mu}{\sqrt{\delta_r}} \sum_{j=1}^n y_{jr}\, w_{ij}^{(t)}.
\end{align*}
Up to scale, these two expressions are equivalent; note that scale is not much important in manifold embedding. Q.E.D.
\end{proof}

\subsection{Out-of-sample Embedding Using Kernel Mapping}

There is a kernel mapping method \cite{gisbrecht2012out,gisbrecht2015parametric} to embed the out-of-sample data in LLE or kernel LLE.
We define a map which maps any data point as $\b{x} \mapsto \b{y}(\b{x})$, where:
\begin{align}\label{equation_kernel_tSNE_map}
\mathbb{R}^p \ni \b{y}(\b{x}) := \sum_{j=1}^n \b{\alpha}_j\, \frac{k(\b{x}, \b{x}_j)}{\sum_{\ell=1}^n k(\b{x}, \b{x}_{\ell})},
\end{align}
and $\b{\alpha}_j \in \mathbb{R}^p$, and $\b{x}_j$ and $\b{x}_{\ell}$ denote the $j$-th and $\ell$-th training data point.
The $k(\b{x}, \b{x}_j)$ is a kernel such as the Gaussian kernel:
\begin{align}
k(\b{x}, \b{x}_j) = \exp(\frac{-||\b{x} - \b{x}_j||_2^2}{2\, \sigma_j^2}),
\end{align}
where $\sigma_j$ is calculated as \cite{gisbrecht2015parametric}:
\begin{align}
\sigma_j := \gamma \times \min_{i}(||\b{x}_j - \b{x}_i||_2),
\end{align}
where $\gamma$ is a small positive number.

Assume we have already embedded the training data points using LLE or kernel LLE; therefore, the set $\{\b{y}_i\}_{i=1}^n$ is available.
If we map the training data points, we want to minimize the following least-squares cost function in order to get $\b{y}(\b{x}_i)$ close to $\b{y}_i$ for the $i$-th training point:
\begin{equation}
\begin{aligned}
& \underset{\b{\alpha}_j\text{'s}}{\text{minimize}}
& & \sum_{i=1}^n ||\b{y}_i - \b{y}(\b{x}_i)||_2^2,
\end{aligned}
\end{equation}
where the summation is over the training data points.
We can write this cost function in matrix form as below:
\begin{equation}\label{equation_kernel_tSNE_leastSquares}
\begin{aligned}
& \underset{\b{A}}{\text{minimize}}
& & ||\b{Y} - \b{K}''\b{A}||_F^2,
\end{aligned}
\end{equation}
where $\mathbb{R}^{n \times p} \ni \b{Y} := [\b{y}_1, \dots, \b{y}_n]^\top$ and $\mathbb{R}^{n \times p} \ni \b{A} := [\b{\alpha}_1, \dots, \b{\alpha}_n]^\top$. 
The $\b{K}'' \in \mathbb{R}^{n \times n}$ is the kernel matrix whose $(i,j)$-th element is defined to be:
\begin{align}
\b{K}''(i,j) := \frac{k(\b{x}_i, \b{x}_j)}{\sum_{\ell=1}^n k(\b{x}_i, \b{x}_{\ell})}.
\end{align}
The Eq. (\ref{equation_kernel_tSNE_leastSquares}) is always non-negative; thus, its smallest value is zero.
Therefore, the solution to this equation is:
\begin{align}
\b{Y} - \b{K}''\b{A} = \b{0} &\implies \b{Y} = \b{K}''\b{A} \nonumber \\
&\overset{(a)}{\implies} \b{A} = \b{K}''^{\dagger}\, \b{Y}, \label{equation_kernel_tSNE_A_matrix}
\end{align}
where $\b{K}''^{\dagger}$ is the pseudo-inverse of $\b{K}''$:
\begin{align}
\b{K}''^{\dagger} = (\b{K}''^\top \b{K}'')^{-1} \b{K}''^\top,
\end{align}
and $(a)$ is because $\b{K}''^{\dagger}\,\b{K}'' = \b{I}$.

Finally, the mapping of Eq. (\ref{equation_kernel_tSNE_map}) for the $n_t$ out-of-sample data points is:
\begin{align}\label{equation_kernel_tSNE_outOfSample_Y}
\b{Y}_t = \b{K}''_t\,\b{A}, 
\end{align}
where the $(i,j)$-th element of the out-of-sample kernel matrix $\b{K}''_t \in \mathbb{R}^{n_t \times n}$ is:
\begin{align}
\b{K}''_t(i,j) := \frac{k(\b{x}_i^{(t)}, \b{x}_j)}{\sum_{\ell=1}^n k(\b{x}_i^{(t)}, \b{x}_{\ell})},
\end{align}
where $\b{x}_i^{(t)}$ is the $i$-th out-of-sample data point, and $\b{x}_j$ and $\b{x}_{\ell}$ are the $j$-th and $\ell$-th training data points.

\section{Incremental LLE}\label{section_incremental_LLE}

Assume that data are online or a stream; hence, data increments by time. Incremental LLE \cite{kouropteva2005incremental} is proposed to handle online data by embedding new received data using the already embedded data. In this sense, it can also be used for out-of-sample embedding. 

Assume we already have $n$ data points; hence, the embedding is obtained by Eq. (\ref{equation_LLE_linearEmbedding_eigenproblem}). As the eigenvectors $\b{Y}$ are orthonormal (so the matrix $\b{Y}$ is orthogonal), Eq. (\ref{equation_LLE_linearEmbedding_eigenproblem}) can be restated as:
\begin{align}\label{equation_LLE_linearEmbedding_eigenproblem_restated}
&\b{Y}^\top \b{M} \b{Y} = (\frac{1}{n}\b{\Lambda}).
\end{align}
Assume we have truncated $\b{Y}$ so we have $p$ eigenvalues, $\b{Y} \in \mathbb{R}^{n \times p}$, $\b{\Lambda} \in \mathbb{R}^{p \times p}$, and $\b{M} \in \mathbb{R}^{n \times n}$.

Suppose $n_t$ new data points are received. Hence, Eq. (\ref{equation_LLE_linearEmbedding_eigenproblem_restated}) becomes:
\begin{align}\label{equation_LLE_linearEmbedding_eigenproblem_incremental_LLE}
&\b{Y}_\text{updated}^\top \b{M}_\text{updated} \b{Y}_\text{updated} = (\frac{1}{n}\b{\Lambda}_\text{updated}),
\end{align}
where $\b{Y}_\text{updated} \in \mathbb{R}^{(n + n_t) \times p}$ and $\b{M}_\text{updated} \in \mathbb{R}^{(n + n_t) \times (n + n_t)}$. Note that as we are considering the smallest eigenvalues when truncating, the eigenvalues in both $\b{\Lambda}_\text{updated}$ and $\b{\Lambda}$ are very small; hence, we can say that we approximately have $\b{\Lambda}_\text{updated} \approx \b{\Lambda}$. Hence, considering Eq. (\ref{equation_LLE_linearEmbedding_eigenproblem_restated}) and the constraints in Eq. (\ref{equation_LLE_linearEmbedding_2}), we have: 
\begin{equation}\label{equation_optimization_incremental_LLE}
\begin{aligned}
& \underset{\b{Y}_\text{updated}}{\text{minimize}}
& & \Big\| \b{Y}_\text{updated}^\top \b{M}_\text{updated} \b{Y}_\text{updated} - \frac{1}{n}\b{\Lambda} \Big\|_F^2, \\
& \text{subject to}
& & \frac{1}{n} \b{Y}_\text{updated}^\top \b{Y}_\text{updated} = \b{I}, \\
& & & \b{Y}_\text{updated}^\top \b{1} = \b{0}.
\end{aligned}
\end{equation}
It is much more efficient than solving Eq. (\ref{equation_LLE_linearEmbedding_2}) for the whole $n + n_t$ data points whose solution is Eq. (\ref{equation_LLE_linearEmbedding_eigenproblem_restated}) and is the eigenvalue problem for an $(n+n_t) \times (n + n_t)$ matrix $\b{M}$. However, Eq. (\ref{equation_optimization_incremental_LLE}) is an optimization over the $p \times p$ matrix within the Frobenius norm. As $p \ll (n + n_t)$, it is much more efficient to use incremental LLE than regular LLE for the whole old and new data. 

This optimization (\ref{equation_optimization_incremental_LLE}) can be solved using the interior point method \cite{boyd2004convex}. After ignoring the second constraint, for the reason explained before, its Lagrangian is \cite{boyd2004convex}: 
\begin{align*}
\mathcal{L} &= \Big\| \b{Y}_\text{updated}^\top \b{M}_\text{updated} \b{Y}_\text{updated} - \frac{1}{n}\b{\Lambda} \Big\|_F^2 \\
&~~~~~~~~ - \textbf{tr}\big(\b{\Lambda}^\top (\frac{1}{n} \b{Y}_\text{updated}^\top \b{Y}_\text{updated} - \b{I})\big).
\end{align*}
According to matrix derivatives and the chain rule, the derivative of this Lagrangian with respect to $\b{Y}_\text{updated}$ is:
\begin{align*}
&\frac{\partial \mathcal{L}}{\partial \b{Y}_\text{updated}} = 2\, (\b{Y}_\text{updated}^\top \b{M}_\text{updated} \b{Y}_\text{updated} - \frac{1}{n}\b{\Lambda} ) \\
&~~~~~~~~~~~~~~~~~~~~ (\b{M}_\text{updated} \b{Y}_\text{updated} + \b{M}_\text{updated}^\top \b{Y}_\text{updated}) \\
&= 4\, (\b{Y}_\text{updated}^\top \b{M}_\text{updated} \b{Y}_\text{updated} - \frac{1}{n}\b{\Lambda} ) \b{M}_\text{updated} \b{Y}_\text{updated}.
\end{align*}
The found $\b{Y}_\text{updated} \in \mathbb{R}^{(n + n_t) \times p}$ by optimization contains the row-wise $p$-dimensional embeddings of both old and new data.

\section{Landmark Locally Linear Embedding for Big Data Embedding}\label{section_landmark_LLE}

LLE is a spectral dimensionality reduction method \cite{saul2006spectral} and its solution follows an eigenvalue problem; see Eq. (\ref{equation_LLE_linearEmbedding_eigenproblem}). Therefore, it cannot handle big data where $n \gg 1$. To tackle this issue, there exist some landmark LLE methods which approximate the embedding of all points using the embedding of some landmarks. In the following, we introduce these methods. 

\subsection{Landmark LLE Using Nystrom Approximation}

Nystrom approximation, introduced below, can be used to make the spectral methods such as LLE scalable and suitable for big data embedding. 

\subsubsection{Nystrom Approximation}

\textit{Nystrom approximation} is a technique used to approximate a positive semi-definite matrix using merely a subset of its columns (or rows) \cite{williams2001using}. 
Consider a positive semi-definite matrix $\mathbb{R}^{n \times n} \ni \b{K} \succeq 0$ whose parts are:
\begin{align}\label{equation_Nystrom_partions}
\mathbb{R}^{n \times n} \ni \b{K} = 
\left[
\begin{array}{c|c}
\b{A} & \b{B} \\
\hline
\b{B}^\top & \b{C}
\end{array}
\right],
\end{align}
where $\b{A} \in \mathbb{R}^{m \times m}$, $\b{B} \in \mathbb{R}^{m \times (n-m)}$, and $\b{C} \in \mathbb{R}^{(n-m) \times (n-m)}$ in which $m \ll n$. 

The Nystrom approximation says if we have the small parts of this matrix, i.e. $\b{A}$ and $\b{B}$, we can approximate $\b{C}$ and thus the whole matrix $\b{K}$. The intuition is as follows. Assume $m=2$ (containing two points, a and b) and $n=5$ (containing three other points, c, d, and e). If we know the similarity (or distance) of points a and b from one another, resulting in matrix $\b{A}$, as well as the similarity (or distance) of points c, d, and e from a and b, resulting in matrix $\b{B}$, we cannot have much freedom on the location of c, d, and e, which is the matrix $\b{C}$. This is because of the positive semi-definiteness of the matrix $\b{K}$. 
The points selected in submatrix $\b{A}$ are named \textit{landmarks}. Note that the landmarks can be selected randomly from the columns/rows of matrix $\b{K}$ and, without loss of generality, they can be put together to form a submatrix at the top-left corner of matrix. 

As the matrix $\b{K}$ is positive semi-definite, by definition, it can be written as $\b{K} = \b{O}^\top \b{O}$. If we take $\b{O} = [\b{R}, \b{S}]$ where $\b{R}$ are the selected columns (landmarks) of $\b{O}$ and $\b{S}$ are the other columns of $\b{O}$. We have:
\begin{align}
\b{K} &= \b{O}^\top \b{O} = 
\begin{bmatrix}
\b{R}^\top \\
\b{S}^\top
\end{bmatrix}
[\b{R}, \b{S}] \label{equation_Nystrom_kernel_OtransposeO} \\
&= 
\begin{bmatrix}
\b{R}^\top \b{R} & \b{R}^\top \b{S} \\
\b{S}^\top \b{R} & \b{S}^\top \b{S}
\end{bmatrix}
\overset{(\ref{equation_Nystrom_partions})}{=} 
\begin{bmatrix}
\b{A} & \b{B} \\
\b{B}^\top & \b{C}
\end{bmatrix}.
\end{align}
Hence, we have $\b{A} = \b{R}^\top \b{R}$. The eigenvalue decomposition \cite{ghojogh2019eigenvalue} of $\b{A}$ gives:
\begin{align}
&\b{A} = \b{U} \b{\Sigma} \b{U}^\top \label{equation_Nystrom_A_eig_decomposition} \\
&\implies \b{R}^\top \b{R} = \b{U} \b{\Sigma} \b{U}^\top \implies \b{R} = \b{\Sigma}^{(1/2)} \b{U}^\top. \label{equation_Nystrom_R}
\end{align}
Moreover, we have $\b{B} = \b{R}^\top \b{S}$ so we have:
\begin{align}
&\b{B} = (\b{\Sigma}^{(1/2)} \b{U}^\top)^\top \b{S} = \b{U} \b{\Sigma}^{(1/2)} \b{S} \nonumber \\
&\overset{(a)}{\implies} \b{U}^\top \b{B} = \b{\Sigma}^{(1/2)} \b{S} \implies \b{S} = \b{\Sigma}^{(-1/2)} \b{U}^\top \b{B}, \label{equation_Nystrom_S}
\end{align}
where $(a)$ is because $\b{U}$ is orthogonal (in the eigenvalue decomposition). 
Finally, we have:
\begin{align}
\b{C} &= \b{S}^\top \b{S} = \b{B}^\top \b{U} \b{\Sigma}^{(-1/2)} \b{\Sigma}^{(-1/2)} \b{U}^\top \b{B} \nonumber \\
&= \b{B}^\top \b{U} \b{\Sigma}^{-1} \b{U}^\top \b{B} \overset{(\ref{equation_Nystrom_A_eig_decomposition})}{=} \b{B}^\top \b{A}^{-1} \b{B}. \label{equation_Nystrom_C}
\end{align}
Therefore, Eq. (\ref{equation_Nystrom_partions}) becomes:
\begin{align}\label{equation_Nystrom_partions_withDetails}
\b{K} \approx 
\left[
\begin{array}{c|c}
\b{A} & \b{B} \\
\hline
\b{B}^\top & \b{B}^\top \b{A}^{-1} \b{B}
\end{array}
\right].
\end{align}

\begin{proposition}
By increasing $m$, the approximation of Eq. (\ref{equation_Nystrom_partions_withDetails}) becomes more accurate. 
If rank of $\b{K}$ is at most $m$, this approximation is exact. 
\end{proposition}
\begin{proof}
In Eq. (\ref{equation_Nystrom_C}), we have the inverse of $\b{A}$. In order to have this inverse, the matrix $\b{A}$ must not be singular. For having a full-rank $\b{A} \in \mathbb{R}^{m \times m}$, the rank of $\b{A}$ should be $m$. This results in $m$ to be an upper bound on the rank of $\b{K}$ and a lower bound on the number of landmarks. In practice, it is recommended to use more number of landmarks for more accurate approximation but there is a trade-off with the speed. 
\end{proof}

\begin{corollary}
As we usually have $m \ll n$, the Nystrom approximation works well especially for the low-rank matrices \cite{kishore2017literature}. Usually, because of the manifold hypothesis, data fall on a submanifold; hence, usually, the kernel (similarity) matrix  or the distance matrix has a low rank. Therefore, the Nystrom approximation works well for many kernel-based or distance-based manifold learning methods. 
\end{corollary}

\subsubsection{Using Kernel Approximation in Landmark LLE}

Consider Eq. (\ref{equation_Nystrom_partions}) or (\ref{equation_Nystrom_partions_withDetails}) as the partitions of the kernel matrix $\b{K}$. Note that the (Mercer) kernel matrix is positive semi-definite so the Nystrom approximation can be applied for kernels.  

Recall that LLE can be viewed as a special case of kernel PCA with the specified kernel in Eq. (\ref{equation_outOfSample_LLE_kernel_K_and_M}). Moreover, recall that according to Eq. (\ref{equation_eigenfunction_decomposition}), the eigenvectors of kernel matrix are used and then Eq. (\ref{equation_embedding_eigenfunction}) embeds data. 
In other words, using the kernel defined by Eq. (\ref{equation_outOfSample_LLE_kernel_K_and_M}), one can apply kernel PCA \cite{ghojogh2019unsupervised} and obtain the desired embedding of LLE. 
However, for big data, the eigenvalue decomposition of kernel matrix is intractable. Therefore, using Eq. (\ref{equation_Nystrom_A_eig_decomposition}), we decompose an $m \times m$ submatrix of kernel. 
In kernel PCA or generalized classical MDS, the kernel can be seen as the inner product of embeddings, i.e. \cite{ghojogh2020multidimensional}:
\begin{align}\label{equation_kernel_innerProduct_Y}
\mathbb{R}^{n \times n} \ni \b{K} = \b{Y}'^\top \b{Y}',
\end{align}
where $\mathbb{R}^{p \times n} \ni \b{Y}' = \b{Y}^\top$ because the embeddings are stacked row-wise in LLE, i.e., $\b{Y} \in \mathbb{R}^{n \times p}$. 
Comparing Eqs. (\ref{equation_kernel_innerProduct_Y}) and (\ref{equation_Nystrom_kernel_OtransposeO}) shows that:
\begin{align}\label{equation_Nystrom_Y}
\mathbb{R}^{n \times n} \ni \b{Y} = [\b{R}, \b{S}] \overset{(a)}{=} [\b{\Sigma}^{(1/2)} \b{U}^\top, \b{\Sigma}^{(-1/2)} \b{U}^\top \b{B}],
\end{align}
where $(a)$ is because of Eqs. (\ref{equation_Nystrom_R}) and (\ref{equation_Nystrom_S}) and the terms $\b{U}$ and $\b{\Sigma}$ are obtained from Eq. (\ref{equation_Nystrom_A_eig_decomposition}). 
The Eq. (\ref{equation_Nystrom_Y}) gives the approximately embedded data, with a good approximation. This is the embedding in landmark LLE using the Nystrom approximation. Truncating this matrix to have $\b{Y}' \in \mathbb{R}^{p \times n}$, with top $p$ rows, gives the $p$-dimensional embedding of the $n$ points, $\mathbb{R}^{n \times p} \ni \b{Y} = \b{Y}'^\top$.

\subsection{Landmark LLE Using Locally Linear Landmarks}

Another way for landmark LLE to handle big data is using Locally Linear Landmarks (LLL) \cite{vladymyrov2013locally}. 
This method maps the $n$ embedded data $\b{Y} \in \mathbb{R}^{n \times p}$ to $m$ landmarks $\widetilde{\b{Y}} \in \mathbb{R}^{m \times p}$, where $m \ll n$, using a projection matrix $\widetilde{\b{U}} = [\widetilde{\b{u}}_1, \dots, \widetilde{\b{u}}_n]^\top \in \mathbb{R}^{n \times m}$:
\begin{align}\label{equation_landmark_LLE_projection_Y}
\mathbb{R}^{n \times p} \ni \b{Y} := \widetilde{\b{U}} \widetilde{\b{Y}}.
\end{align}

Assume that in some way, we choose the landmarks in the input space. For example, we choose a subset of data points $\b{X} \in \mathbb{R}^{d \times n}$ to have the landmarks $\widetilde{\b{X}} \in \mathbb{R}^{d \times m}$. In other words, $\mathbb{C}\text{ol}(\b{X}) \subseteq \mathbb{C}\text{ol}(\widetilde{\b{X}})$ where $\mathbb{C}\text{ol}(\cdot)$ denotes the column space of matrix. 
The projection to landmarks should also work for the input space as:
\begin{align}
\mathbb{R}^{n \times d} \ni \b{X}^\top := \widetilde{\b{U}} \widetilde{\b{X}}^\top.
\end{align}
With adding some constraint, we can write this goal as an optimization problem:
\begin{equation}
\begin{aligned}
& \underset{\widetilde{\b{U}}}{\text{minimize}}
& & \sum_{i=1}^n ||\b{x}_i - \widetilde{\b{X}} \widetilde{\b{u}}_i||_2^2., \\
& \text{subject to}
& & \b{1}^\top \widetilde{\b{u}}_i = 1, ~~~ \forall i \in \{1, \dots, n\},
\end{aligned}
\end{equation}
which is exactly in the form of Eq. (\ref{equation_LLE_linearReconstruct}). Hence, its solution is similar to Eq. (\ref{equation_w_tilde_solution}):
\begin{align}
\mathbb{R}^{m} \ni \widetilde{\b{u}}_i =  \frac{\widetilde{\b{G}}_i^{-1} \b{1}}{\b{1}^\top \widetilde{\b{G}}_i^{-1} \b{1}},
\end{align}
where:
\begin{align}
\mathbb{R}^{m \times m} \ni \widetilde{\b{G}}_i := (\b{x}_i \b{1}^\top - \widetilde{\b{X}})^\top (\b{x}_i \b{1}^\top - \widetilde{\b{X}}).
\end{align}

Also according to Eq. (\ref{equation_landmark_LLE_projection_Y}), the Eq. (\ref{equation_LLE_linearEmbedding_2}) becomes:
\begin{equation}\label{equation_landmark_LLE_linearEmbedding}
\begin{aligned}
& \underset{\widetilde{\b{Y}}}{\text{minimize}}
& & \textbf{tr}(\widetilde{\b{Y}}^\top \widetilde{\b{U}}^\top \b{M} \widetilde{\b{U}} \widetilde{\b{Y}}), \\
& \text{subject to}
& & \frac{1}{n} \widetilde{\b{Y}}^\top \widetilde{\b{U}}^\top \widetilde{\b{U}} \widetilde{\b{Y}} = \b{I}, 
\end{aligned}
\end{equation}
whose second constraint is ignored because, as explained before, it is satisfied anyways. Let:
\begin{align}
\mathbb{R}^{m \times m} \ni \widetilde{\b{M}} := \widetilde{\b{U}}^\top \b{M} \widetilde{\b{U}}. 
\end{align}
Similar to solution of Eq. (\ref{equation_LLE_linearEmbedding_2}), the solution to Eq. (\ref{equation_landmark_LLE_linearEmbedding}) is the eigenvalue problem for $\widetilde{\b{M}}$ \cite{ghojogh2019eigenvalue}. 
In other words, the embeddings of landmark points, $\widetilde{\b{Y}}$, are the $p$ smallest eigenvectors of $\widetilde{\b{M}}$ after ignoring the eigenvector with zero eigenvalue. As the dimensionality of $\widetilde{\b{M}}$ is $m \times m$, landmark LLE using LLL is much more efficient than LLE whose embeddings are the eigenvectors of $\b{M} \in \mathbb{R}^{n \times n}$. The difference of efficiency gets noticeable especially for big data where $n \gg m$. 
Finally, using Eq. (\ref{equation_landmark_LLE_projection_Y}), the embeddings of all $n$ points are approximated by the obtained embeddings of $m$ landmarks.

\section{Parameter Selection of the Number of Neighbors in LLE}\label{section_parameter_selection_LLE}

LLE has a hyper-parameter which is the number of neighbors $k$. 
There are several different algorithms for finding an optimal $k$. In the following, we explain these algorithms. 

\subsection{Parameter Selection Using Residual Variance}\label{section_parameter_selection_residual_variance}

Assume we have candidate number of neighbors, denoted by $\{1, 2, \dots, k_\text{max}\}$ which we want to find the best $k$ from. For every $k \in \{1, 2, \dots, k_\text{max}\}$, we can run LLE and find the embeddings $\b{Y}$ for data $\b{X}$. Let $\b{D}_X$ and $\b{D}_Y$ denote the Euclidean distance matrices over $\b{X}$ and $\b{Y}$, respectively. 
Let $\rho^2_{\b{D}_X, \b{D}_Y}$ be the standard linear correlation coefficient, i.e., $\rho^2_{\b{D}_X, \b{D}_Y} := S_{\b{D}_X, \b{D}_Y} / (S_{\b{D}_X} S_{\b{D}_Y})$ where $\rho^2_{\b{D}_X, \b{D}_Y}$ is the covariance of $\b{D}_X$ and $\b{D}_Y$ and $S_{\b{D}_X}$ and $S_{\b{D}_Y}$ are the standard deviations of $\b{D}_X$ and $\b{D}_Y$, respectively. 
The residual variance for a number of neighbors $k$ is defined as \cite{kouropteva2002selection}: 
\begin{align}
\sigma^2_k(\b{D}_X, \b{D}_Y) := 1 - \rho^2_{\b{D}_X, \b{D}_Y}.
\end{align}
The $k$ value giving the smallest value for the residual variance is the optimal number of neighbors because it maximizes the correlation between the distances in the input and embedding spaces. Hence:
\begin{align}\label{equation_selection_k_residual_variance_method}
k := \arg \min_k \sigma^2_k(\b{D}_X, \b{D}_Y).
\end{align}

In order not to run LLE for all $k \in \{1, 2, \dots, k_\text{max}\}$, which is computationally expensive, we can have a hierarchical approach \cite{kouropteva2002selection}. In this approach, we calculate $\varepsilon(\widetilde{\b{W}})$, in Eq. (\ref{equation_LLE_linearReconstruct}), for every value of $k \in \{1, 2, \dots, k_\text{max}\}$. For the local minimums of $\varepsilon(\widetilde{\b{W}})$ (whenever $\varepsilon(\widetilde{\b{W}})$ for a $k$ is smaller than that for $k-1$ and $k+1$), we calculate Eq. (\ref{equation_selection_k_residual_variance_method}) and find the best $k$ among the $k$'s corresponding to local minimums. 

\subsection{Parameter Selection Using Procrustes Statistics}

Another method for parameter selection of $k$ in LLE is \cite{goldberg2009local} which uses Procrustes statistics \cite{sibson1978studies}. The Procrustes statistics between $\b{X} = [\b{x}_1, \dots, \b{x}_n] \in \mathbb{R}^{d \times n}$ and their embeddings $\b{Y} = [\b{y}_1, \dots, \b{y}_n]^\top \in \mathbb{R}^{n \times p}$ is \cite{sibson1978studies,goldberg2009local}:
\begin{align}
P(\b{X}, \b{Y}) &:= \sum_{i=1}^n \|\b{x}_i - \b{y}_i \b{A}^\top - \b{b}\|_2^2 \nonumber \\
&= \|\b{H}_n (\b{X}^\top - \b{Y} \b{A}^\top)\|_F^2,
\end{align}
with the orthogonal rotation matrix, i.e. $\b{A}^\top \b{A} = \b{I}$, and the translation matrix $\b{b} = \bar{\b{x}} - \bar{\b{y}} \b{A}^\top$ where $\bar{\b{x}}$ and $\bar{\b{y}}$ are the means of samples $\b{X}$ and $\b{Y}$, respectively. The matrix $\mathbb{R}^{n \times n} \ni \b{H}_n = \b{I}_n - (1/n) \b{1}\b{1}^\top$ is the centering matrix. 
According to the Procrustes statistics \cite{sibson1978studies}, the rotation matrix can be computed by $\mathbb{R}^{d \times p} \ni \b{A} = \b{U} \b{V}^\top$ where $\b{U}\b{\Sigma}\b{V}^\top$ is the singular value decomposition of $\b{X} \b{H}_n \b{Y} \in \mathbb{R}^{d \times p}$. 

Let $\b{X}_i \in \mathbb{R}^{d \times k}$ and $\b{Y}_i \in \mathbb{R}^{k \times p}$ be the $k$ neighbors of $\b{x}_i$ in the input and embedding spaces, respectively. 
For every $k$, we apply LLE and get some embedding $\b{Y}$ for $\b{X}$, as well as some neighborhood graph. A normalized Procrustes statistics for a number of neighbors $k$ is \cite{goldberg2009local}:
\begin{align}
R_k(\b{X}, \b{Y}) := \frac{1}{n} \sum_{i=1}^n \frac{P(\b{X}_i, \b{Y}_i)}{\|\b{H}_k \b{X}_i^\top\|_F^2},
\end{align}
The best $k \in \{1. \dots, k_\text{max}\}$ reduces Procrustes statistics the most:
\begin{align}\label{equation_selection_k_Procrustes_stat_method}
k := \arg \min_k R_k(\b{X}, \b{Y}).
\end{align}
Again, a hierarchical approach, introduced in Section \ref{section_parameter_selection_residual_variance}, can be used to determine the best value $k$ using Eq. (\ref{equation_selection_k_Procrustes_stat_method}).

\subsection{Parameter Selection Using Preservation Neighborhood Error}

Consider the data points $\{\b{x}_i \in \mathbb{R}^d\}_{i=1}^n$ and their embeddings $\{\b{y}_i \in \mathbb{R}^p\}_{i=1}^n$. For a point $\b{x}_i$, let its $k$ neighbors in the input space be denoted by $\{\b{\eta}_i \in \mathbb{R}^d\}_{i=1}^k$. The embeddings of $\{\b{\eta}_i\}_{i=1}^k$ are denoted by $\{\b{\phi}_i \in \mathbb{R}^p\}_{i=1}^k$. Now, let the $k$ neighbors of $\b{y}_i$ in the embedding space space be denoted by $\{\b{\beta}_i \in \mathbb{R}^p\}_{i=1}^k$. The points which are among the $k$ neighbors of $\b{y}_i$ but not among the $k$ neighbors of $\b{x}_i$ are denoted by $\{\b{\gamma}_i \in \mathbb{R}^p\}_{i=1}^{k_i'}$ in the embedding space, where the number of these points is denoted by $k'_i$. In other words, we have $\{\gamma_i\}_{i=1}^{k_i'} = \{\b{\beta}_i\}_{i=1}^k - \{\b{\phi}_i\}_{i=1}^k$.
The corresponding points to $\{\b{\gamma}_i \in \mathbb{R}^p\}$ in the input space are denoted by $\{\b{\theta}_i \in \mathbb{R}^d\}_{i=1}^{k_i'}$. For an illustration of these definitions, the reader can refer to {\citep[Fig. 1]{alvarez2011global}}. 

The Preservation Neighborhood Error (PNE), for a number of neighbors $k$, is defined as \cite{alvarez2011global}:
\begin{equation}
\begin{aligned}
&\text{PNE}_k(\b{X}, \b{Y}) \\
&:= \frac{1}{2n} \sum_{i=1}^n \Big( \sum_{j=1}^k \frac{(\|\b{x}_i - \b{\eta}_j\|_2 - \|\b{y}_i - \b{\phi}_j\|_2)^2}{k} + \\
&~~~~~~~~~~~~~~~~~~~~~ \sum_{j=1}^{k'_i} \frac{(\|\b{x}_i - \b{\theta}_j\|_2 - \|\b{y}_i - \b{\gamma}_j\|_2)^2}{k'_i} \Big).
\end{aligned}
\end{equation}
The first term in summation tries to preserve the local structure of points in the embedding space as in the input space. The second term tries to keep the points away in the embedding space if they are far from each other in the input space; in other words, the second term avoids false folding of manifold. 
The best $k \in \{1. \dots, k_\text{max}\}$ reduces PNE the most:
\begin{align}\label{equation_selection_k_PNE_method}
k := \arg \min_k \text{PNE}_k(\b{X}, \b{Y}).
\end{align}
Again, a hierarchical approach, introduced in Section \ref{section_parameter_selection_residual_variance}, can be used to determine the best value $k$ using Eq. (\ref{equation_selection_k_PNE_method}). 

\subsection{Parameter Selection Using Local Neighborhood Selection}

There is another algorithm for selecting the best number of neighbors, named Local Neighborhood Selection (LNS) \cite{alvarez2011global}. This algorithm finds the best number of neighbors per each point $\b{x}_i$; therefore, it allows us to have different number of neighbors for different points. In this algorithm, we first calculate the Euclidean and geodesic distance matrices, denoted by $\b{D} \in \mathbb{R}^{n \times n}$ and $\b{D}^{(g)} \in \mathbb{R}^{n \times n}$, respectively. 
Initialize $k_\text{min} = 1$. 
We find the $k_\text{min}$-NN graph using $\b{D}$. 
We check if the $k_\text{min}$-NN graph is connected, using a Breadth First Search (BFS) \cite{cormen2009introduction}. If it is not connected, we increment $k_\text{min}$ by one. We do this until the graph gets connected. We set $k_\text{max} := n^2 / (k_\text{min} \times |E|)$ where $|E|$ is the number of edges in the $k_\text{min}$-NN graph. We define $\b{k} = [\b{k}(1), \dots, \b{k}(k_\text{max} - k_\text{min})] := [k_\text{min}+1, \dots, k_\text{max}] \in \mathbb{R}^{k_\text{max} - k_\text{min}}$. 
Let $\b{\eta}^{D,k}_i$ and $\b{\eta}^{D^{(g)},k}_i$ be the set of $k$-NN of $\b{x}_i$ using the distance matrices $\b{D}$ and $\b{D}^{(g)}$, respectively, where $k \in \{k_\text{min}+1, \dots, k_\text{max}\}$. 
If $|\cdot|$ and $\bar{\cdot}$ denote the cardinality and complement of set, respectively, the $(i,j)$-th element of the linearity conservation matrix $\b{V} \in \mathbb{R}^{n \times (k_\text{max} - k_\text{min})}$ is:
\begin{align}
\b{V}(i,j) := \frac{\big|\overline{(\b{\eta}^{D,\b{k}(j)}_i \cap \b{\eta}^{D^{(g)},\b{k}(j)}_i)}\big|}{\b{k}(j)}.
\end{align}
The smaller this quantity, the closer the geodesic and Euclidean distances behave so the more local structure is preserved. For every row of the linearity conservation matrix (i.e., for every point $\b{x}_i$), the best number of neighbors is determined as:
\begin{align}
k(\b{x}_i) := \arg \min_{\b{k}(j)} \b{V}(i,j). 
\end{align}
In case of ties, we get the largest value of $\b{k}(j)$ for better capture of neighborhood structure.  

\section{Supervised and Semi-Supervised LLE}\label{section_supervised_LLE}

In supervised and semi-supervised LLE, the class labels are used fully or partially, respectively. 
There are different versions of these methods which are explained in the following. We do not explain SLLEP \cite{li2011supervised}, as a supervised LLE method, here because it was introduced in Section \ref{section_LLE_projection}. Moreover, we do not cover supervised LLE by adjusting weights \cite{he2019nonlinear} and Discriminant LLE \cite{li2008discriminant} here because they will be explained in Sections \ref{section_adjusted_weights_LLE} and \ref{section_discriminant_LLE}, respectively. 

\subsection{Supervised LLE}\label{section_SLLE}

We can have Supervised LLE (SLLE) \cite{kouropteva2002beyond,de2003supervised,kouropteva2003supervised}, which can be useful for both embedding and classification \cite{de2002locally}. 
SLLE makes use of class labels of the data points. 
The main idea of SLLE is to artificially increase the inter-class variance of data by adding to the distances of points from different classes. Assume the Euclidean distance matrix is denoted by $\b{D} \in \mathbb{R}^{n \times n}$. In SLLE, the distance matrix is modified to \cite{de2003supervised}:
\begin{align}\label{equation_SLLE_modified_distance}
\mathbb{R}^{n \times n} \ni \b{D}' := \b{D} + \alpha\, (d_\text{max}) (\b{1}\b{1}^\top - \b{\Delta}),
\end{align}
where $\b{1}\b{1}^\top \in \mathbb{R}^{n \times n}$ is the matrix with all elements as one, $d_\text{max} \in \mathbb{R}$ is the diameter of data:
\begin{align}
d_\text{max} := \max_{i,j}(\|\b{x}_i - \b{x}_j\|_2),
\end{align}
and $\b{\Delta}$ is a matrix whose $(i,j)$-th element is:
\begin{align}
\b{\Delta}(i,j) :=
\left\{
    \begin{array}{ll}
        1 & \mbox{if } c_i = c_j, \\
        0 & \mbox{Otherwise,}
    \end{array}
\right.
\end{align}
where $c_i$ denotes the class label of $\b{x}_i$, and $\alpha \in [0,1]$. When $\alpha=0$, SLLE is reduced to LLE which is unsupervised. When $\alpha=1$, SLLE is fully supervised; this case is also named 1-SLLE \cite{kouropteva2002beyond}. When $\alpha \in (0,1)$, we have partially supervised SLLE, also called $\alpha$-SLLE \cite{de2002locally}. Note that Eq. (\ref{equation_SLLE_modified_distance}) does not change the distances between points belonging to the same class. 
After modifying the distance matrix, SLLE finds $k$NN graph using the modified distances and the rest of algorithm is the same as in LLE. 

\subsection{Enhanced Supervised LLE}

Enhanced Supervised LLE (ESLLE) \cite{zhang2009enhanced}, not only artificially increases the inter-class variances, but also artificially reduces the intra-class variances. Note that the idea of increasing and decreasing the inter-class and intra-class variances, respectively, is common in supervised embedding, such as Fisher discriminant analysis \cite{ghojogh2019fisher}. ESLLE modifies the distances to:
\begin{align}\label{equation_ESLLE_distance_modification}
\b{D}' := 
\left\{
    \begin{array}{ll}
        \sqrt{1 - e^{-\b{D}^2/\beta}} & \mbox{if } c_i = c_j, \\
        \sqrt{e^{\b{D}^2 / \beta}} - \alpha & \mbox{Otherwise,}
    \end{array}
\right.
\end{align}
where $\alpha \in [0,1]$ and:
\begin{align}
\beta := \text{average}_{i,j}(\|\b{x}_i - \b{x}_j\|_2).
\end{align}
In ESLLE, the distance of points from different classes grows exponentially while the distances of points in the same class have a horizontal asymptote of one (see {\citep[Fig. 1]{zhang2009enhanced}}). 
Using the modified distances, $k$NN graph is found and the rest is as in LLE.

\subsection{Supervised LLE Projection}\label{section_LLE_projection}

We can approximate the mapping $\b{X} \mapsto \b{Y}$ using a linear projection. Supervised LLE Projection (SLLEP) \cite{li2011supervised} finds a linear projection in the context of SLLE \cite{de2003supervised}. First, SLLEP finds the embedding of training data, $\b{Y}$, using SLLE, introduced in Section \ref{section_SLLE}. It then tries to approximate this embedding by a linear projection $\b{Y} = \b{U}^\top \b{X}$ where $\b{U} = [\b{u}_1, \dots, \b{u}_p] \in \mathbb{R}^{d \times p}$ is the projection matrix. Let the embedding of point $\b{x}_i$ be $\mathbb{R}^p \ni \b{y}_i := [\b{y}_i(1), \dots, \b{y}_i(p)]^\top$. Also, let $\mathbb{R}^n \ni \b{y}^j := [\b{y}_1(j), \dots, \b{y}_n(j)]^\top$. This approximation can be done using least squares optimization:
\begin{align}
\b{u}_j = \arg \min_{\b{u}} \sum_{i=1}^{n} (\b{u}^\top \b{x}_i - \b{y}_i(j))^2, \quad \forall j \in \{1, \dots, p\},
\end{align}
whose solution is similar to the solution of linear regression \cite{hastie2009elements}:
\begin{align}
\b{u}_j = (\b{X}\b{X}^\top)^{-1} \b{X} \b{y}^j.
\end{align}
In case $\b{X}\b{X}^\top$ is singular, we can use the regularized least squares optimization with the regularization parameter $\beta$. In this case, the solution is similar to the ridge regression \cite{hastie2009elements}:
\begin{align}
\b{u}_j = (\b{X}\b{X}^\top + \beta \b{I})^{-1} \b{X} \b{y}^j.
\end{align}
SLLEP can be used for approximation of out-of-sample embedding for new data $\b{X}^{(t)}$ by $\b{Y}^{(t)} = \b{U}^\top \b{X}^{(t)}$. 
It is also noteworthy that the approximation used in SLLEP can be used for approximating the unsupervised LLE with a linear projection, too. 

\subsection{Probabilistic Supervised LLE}

Probability-based LLE (PLLE) \cite{zhao2009supervised} is another supervised method for LLE which can also handle out-of-sample data. 
For every training point $\b{x}_i$, the probability of belonging to class $c_i$ should be one; hence, its one-hot encoding is:
\begin{align}
\mathbb{R}^{c} \ni \b{p}(\b{x}_i) := \b{1}_{c_i} = [0, \dots, 0,1,0, \dots, 0]^\top,
\end{align}
whose $c_i$-th element is one.
However, for the out-of-sample data, the probability is found using logistic regression \cite{kleinbaum2002logistic}. PLLE, first, applies unsupervised LLE on both training and out-of-sample data (see Sections \ref{section_LLE} and \ref{section_LLE_outOfSample}). Then, for the embedding of out-of-sample points, denoted by $\{\b{y}_i^{(t)}\}_{i=1}^{n_t}$, it learns logistic functions of all $c$ classes:
\begin{align}
\pi(\b{y}_i^{(t)}; \b{a}_\ell, \b{b}_\ell) := \frac{e^{\b{a}_\ell + \b{b}_\ell^\top \b{y}_i^{(t)}}}{1 + e^{\b{a}_\ell + \b{b}_\ell^\top \b{y}_i^{(t)}}}, \quad \forall \ell \in \{1, \dots, c\},
\end{align}
where the parameters $\b{a}_\ell$ and $\b{b}_\ell$ are found by logistic regression. Hence, we have $\{\pi(\b{y}_i^{(t)}; \b{a}_\ell, \b{b}_\ell)\}_{i=1}^c$. 
The probability of $\b{x}_i$ belonging to every $\ell$-th class is:
\begin{align}
\mathbb{R} \ni p_\ell(\b{x}_i^{(t)}) := \frac{\pi(\b{y}_i^{(t)}; \b{a}_\ell, \b{b}_\ell)}{\sum_{\ell'=1}^c \pi(\b{y}_i^{(t)}; \b{a}_{\ell'}, \b{b}_{\ell'})}.
\end{align}
Therefore, the probability vector for $\b{x}_i^{(t)}$ is $\mathbb{R}^{c} \ni \b{p}(\b{x}_i^{(t)}) := [p_\ell(\b{x}_1^{(t)}), \dots, p_\ell(\b{x}_c^{(t)})]^\top$. 
PLLE uses Eq. (\ref{equation_SLLE_modified_distance}) for modification of distances but, as it can handle out-of-sample data, we put together all training and out-of-sample points in this stage; thus, we have have $\b{D}, \b{D}', \b{\Delta} \in \mathbb{R}^{(n + n_t) \times (n + n_t)}$ and define $\b{\Delta}(i,j)$ as ($\forall i, j \in \{1, \dots, n+n_t\}$):
\begin{align}
\b{\Delta}(i,j) :=
\left\{
    \begin{array}{ll}
        1 & \mbox{if } c_i = c_j, \\
        \b{p}(\b{x}_i)^\top \b{p}(\b{x}_j) & \mbox{Otherwise.}
    \end{array}
\right.
\end{align}
Again, using the modified distances, $k$NN graph is found and the rest is as in LLE.

\subsection{Semi-Supervised LLE}

When some of data have labels and some don not, we can use semi-supervised LLE \cite{zhang2009dimension}. Similar to Eq. (\ref{equation_ESLLE_distance_modification}), this method modifies the distances as:
\begin{align}
\b{D}' \! := \!
\left\{
    \begin{array}{ll}
        \sqrt{1 - e^{-\b{D''}^2/\beta}} - \alpha & \mbox{if } c_i = c_j, \\
        \sqrt{1 - e^{-\b{D''}^2/\beta}} & \mbox{if } \b{x}_i \text{ or } \b{x}_j \text{ is unlabeled}, \\
        \sqrt{e^{\b{D''}^2 / \beta}} & \mbox{Otherwise,}
    \end{array}
\right.
\end{align}
where:
\begin{align}
\b{D}''(i,j) := \frac{\b{D}(i,j)}{ \sqrt{m_i \times m_j} },
\end{align}
and $\mathbb{R} \ni m_i := \text{average}_\ell(\|\b{x}_i - \b{x}_\ell)\|_2; \forall \ell \in \{1, \dots, n\})$. 

\subsection{Supervised Guided LLE}

There is a supervised LLE method, named Guided LLE (GLLE) \cite{alipanahi2011guided}, which makes use of Hilbert-Schmidt Independence Criterion (HSIC) \cite{gretton2005measuring} for utilizing the labels in embedding. 
In the following, we explain this method. 

\subsubsection{Seeing LLE as Kernel PCA}

As was mentioned in Section \ref{section_outOfSample_eigenfunctions}, LLE can be seen as a special case of kernel LLE where the inverse or negative sign of $\b{M}$ can be interpreted as its kernel. The kernel of LLE in kernel PCA can be either Eq. (\ref{equation_outOfSample_LLE_kernel_K_and_M}) \cite{scholkopf2002learning,bengio2003learning} or \cite{ham2004kernel}:
\begin{align}\label{equation_LLE_asKernelPCA_inverseM}
\mathbb{R}^{n \times n} \ni \b{K} := \b{M}^\dagger,
\end{align}
which is the pseudo-inverse of the matrix $\b{M}$. 

\subsubsection{Hilbert-Schmidt Independence Criterion}

Suppose we want to measure the dependence of two random variables. Measuring the correlation between them is easier because correlation is just ``linear'' dependence. 
According to \cite{hein2004kernels}, two random variables are independent if and only if any bounded continuous functions of them are uncorrelated. Therefore, if we map the two random variables $\b{x}$ and $\b{y}$ to two different (``separable'') Reproducing Kernel Hilbert Spaces (RKHSs) and have $\b{\phi}(\b{x})$ and $\b{\phi}(\b{y})$, we can measure the correlation of $\b{\phi}(\b{x})$ and $\b{\phi}(\b{y})$ in the Hilbert space to have an estimation of dependence of $\b{x}$ and $\b{y}$ in the original space. 

The correlation of $\b{\phi}(\b{x})$ and $\b{\phi}(\b{y})$ can be computed by the Hilbert-Schmidt norm of the cross-covariance of them \cite{gretton2005measuring}. Note that the squared Hilbert-Schmidt norm of a matrix $\b{A}$ is \cite{bell2016trace}:
\begin{align*}
||\b{A}||_{HS}^2 := \textbf{tr}(\b{A}^\top \b{A}),
\end{align*}
and the cross-covariance matrix of two vectors $\b{x}$ and $\b{y}$ is \cite{gretton2005measuring}:
\begin{align*}
\mathbb{C}\text{ov}(\b{x}, \b{y}) := \mathbb{E}\Big[&\big(\b{x} - \mathbb{E}(\b{x})\big) \big(\b{y} - \mathbb{E}(\b{y})\big) \Big].
\end{align*}
Using the explained intuition, an empirical estimation of the HSIC is introduced \cite{gretton2005measuring}:
\begin{align}\label{equation_HSIC}
\text{HSIC} := \frac{1}{(n-1)^2}\, \textbf{tr}(\b{K}_x\b{H}\b{K}_y\b{H}),
\end{align}
where $\b{K}_x$ and $\b{K}_y$ are the kernels over $\b{x}$ and $\b{y}$, respectively, and $\b{H}$ is the centering matrix.
The term $1/(n-1)^2$ is used for normalization.

The HSIC (Eq. (\ref{equation_HSIC})) measures the dependence of two random variable vectors $\b{x}$ and $\b{y}$. Note that $\text{HSIC}=0$ and $\text{HSIC}>0$ mean that $\b{x}$ and $\b{y}$ are independent and dependent, respectively. The greater the HSIC, the more dependence they have.

\subsubsection{Interpreting LLE using HSIC}\label{section_interpret_LLE_using_HSIC}

Suppose we consider the kernel $\b{K}_x$ in HSIC to be Eq. (\ref{equation_LLE_asKernelPCA_inverseM}) and its other kernel to be a linear kernel, i.e., $\b{K}_y := \b{Y} \b{Y}^\top$ (note that the embedded points are stacked in $\b{Y}$ row-wise). 
We want to maximize the HSIC to have large dependence between the data $\b{X}$ and their embedding $\b{Y}$. This maximization can be modeled by a constrained optimization problem:
\begin{equation}
\begin{aligned}
& \underset{\b{Y}}{\text{maximize}}
& & \textbf{tr}(\b{Y}^\top\b{H}\b{M}^\dagger\b{H}\b{Y}) \overset{(a)}{=} \textbf{tr}(\b{Y}^\top\b{M}^\dagger\b{Y}), \\
& \text{subject to}
& & \frac{1}{n} \b{Y}^\top \b{Y} = \b{I},
\end{aligned}
\end{equation}
where $(a)$ is because the matrix $\b{M}$ is already double-centered \cite{alipanahi2011guided}. 
This maximization problem can be converted to a minimization problem as:
\begin{equation}\label{equation_GLLE_optimization_embedding}
\begin{aligned}
& \underset{\b{Y}}{\text{minimize}}
& & \textbf{tr}(\b{Y}^\top\b{M}\b{Y}), \\
& \text{subject to}
& & \frac{1}{n} \b{Y}^\top \b{Y} = \b{I},
\end{aligned}
\end{equation}
which is equivalent to Eq. (\ref{equation_LLE_linearEmbedding_2}), ignoring the second constraint in Eq. (\ref{equation_LLE_linearEmbedding_2}) which is already satisfied. This shows that the optimization of embedding in LLE can be seeing as maximizing the HSIC (or dependence) between the input and embedding data. 

\subsubsection{Guiding LLE Using Labels}

For discrimination of classes, consider maximization of dependence between a linear kernel over embedding and a kernel over class labels (targets), denoted by $\b{K}_t$:
\begin{equation}\label{equation_GLLE_optimization_labels}
\begin{aligned}
& \underset{\b{Y}}{\text{maximize}}
& & \textbf{tr}(\b{Y}^\top\b{H}\b{K}_t\b{H}\b{Y}), \\
& \text{subject to}
& & \frac{1}{n} \b{Y}^\top \b{Y} = \b{I},
\end{aligned}
\end{equation}
which can be also converted to a minimization problem using the pseudo-inverse of $\b{K}_t$. The kernel over labels can be a delta kernel \cite{barshan2011supervised,ghojogh2019unsupervised}. 
After converting Eq. (\ref{equation_GLLE_optimization_labels}) to minimization, we can combine Eqs. (\ref{equation_GLLE_optimization_embedding}) and (\ref{equation_GLLE_optimization_labels}) as:
\begin{equation}
\begin{aligned}
& \underset{\b{Y}}{\text{minimize}}
& & \textbf{tr}\big(\b{Y}^\top((1-\alpha)\b{M} + \alpha \b{K}_t)\b{Y}\big), \\
& \text{subject to}
& & \frac{1}{n} \b{Y}^\top \b{Y} = \b{I},
\end{aligned}
\end{equation}
where $\alpha \in [0,1]$.
The solution to this optimization problem is the smallest $p$ eigenvectors of $(1-\alpha)\b{M} + \alpha \b{K}_t$ \cite{ghojogh2019eigenvalue} after ignoring the first eigenvector with eigenvalue zero. 
Note that this optimization guides LLE to have an embedding with more discrimination of classes.

\section{Robust Locally Linear Embedding}\label{section_robust_LLE}

In presence of outliers and noise, LLE cannot preserve the local structure of manifold well enough because some bias is introduced to reconstruction of points by the outliers \cite{chang2006robust}. 
Therefore, Robust LLE (RLLE) is proposed to handle outliers in LLE. There exist at least two methods for RLLE which we explain in the following.

\subsection{Robust LLE Using Least Squares Problem}

One approach for RLLE is using least squares problem to handle noise \cite{chang2006robust}. 
This RLLE uses an iterative optimization approach \cite{jain2017non} where it iterates between Principal Component Analysis (PCA) and finding reliability weights. In every iteration, for every point $\b{x}_i$, it minimizes the weighted reconstruction error using PCA \cite{ghojogh2019unsupervised} by a least squares problem:
\begin{equation}\label{equation_RLLE_PCA}
\begin{aligned}
& \underset{\b{U}_i}{\text{minimize}}
& & \sum_{j=1}^k a_{ij}\, e_{ij} := \sum_{j=1}^k a_{ij}\, ||\b{x}_{ij} - \b{b}_i - \b{U}_i\,\b{y}_{ij}||_2^2,
\end{aligned}
\end{equation}
where $\b{b}_i \in \mathbb{R}^d$ and $\b{U}_i \in \mathbb{R}^{d \times p}$ are the bias and PCA projection matrix, respectively, $\b{y}_{ij} \in \mathbb{R}^{p}$ is the embedding of $\b{x}_{ij}$, and $\{a_{ij}\}_{j=1}^k$ are the reliability weights. 
The solution to this optimization is \cite{chang2006robust}:
\begin{align}\label{equation_RLLE_b}
& \b{b}_i := \frac{\sum_{j=1}^k a_{ij}\, \b{x}_{ij}}{\sum_{j=1}^k a_{ij}}, 
\end{align}
and the columns of $\b{U}_i$ are the top $p$ eigenvectors of the covariance matrix over the neighbors:
\begin{align}\label{equation_RLLE_S}
\b{S}_i := \frac{1}{k} \sum_{j=1}^k a_{ij}\, (\b{x}_{ij} - \b{b}_i) (\b{x}_{ij} - \b{b}_i)^\top.
\end{align}
Then, the weights $\{a_j\}_{j=1}^k$ are obtained inspired by the Huber function as \cite{chang2006robust}:
\begin{align}\label{equation_RLLE_a}
a_{ij} := 
\left\{
    \begin{array}{ll}
        1 & \mbox{if } e_{ij} \leq c_i, \\
        c_i/e_j & \mbox{if } e_{ij} > c_i,
    \end{array}
\right.
\end{align}
where $e_{ij}$ is defined in Eq. (\ref{equation_RLLE_PCA}) and $c_i$ is the mean error residual, i.e., $c_i := (1/k) \sum_{j=1}^k e_{ij}$. 
Using an iterative approach, or Iteratively Reweighted Least Squares (IRLS) \cite{holland1977robust}, $\b{b}_i$, $\b{U}_i$, and $\{a_{ij}\}_{j=1}^k$ are fine tuned for the $k$ neighbors of every point $\b{x}_i$, by Eqs. (\ref{equation_RLLE_b}), (\ref{equation_RLLE_S}), and (\ref{equation_RLLE_a}). In this way, the reliability weights $\{a_{ij}\}_{j=1}^k$ are calculated for every point. Let the mean reliability weights over the neighbors of a point determines the reliability weight of that point. We calculate it as $s_i := (1/k) \sum_{j=1}^k a_{ij}$. Then, RLLE weights the objective of Eq. (\ref{equation_LLE_linearEmbedding}) as \cite{chang2006robust}:
\begin{equation}
\begin{aligned}
& \underset{\b{Y}}{\text{minimize}}
& & \sum_{i=1}^n s_i \Big|\Big|\b{y}_i - \sum_{j=1}^n w_{ij} \b{y}_j\Big|\Big|_2^2,
\end{aligned}
\end{equation}
with the constraints in Eq. (\ref{equation_LLE_linearEmbedding}). Hence, the embeddings are weighted to be robust to outliers. 

\subsection{Robust LLE Using Penalty Functions}

Another method for RLLE uses penalty function for regularized optimization \cite{winlaw2011robust}. In presence of noise or outliers, some weights of reconstruction of a point by its neighbors explode because the distance of outliers from other points is usually large. The paper \cite{winlaw2011robust} proposes two different penalty functions for RLLE, explained in the following. 

\subsubsection{RLLE with $\ell_2$ Norm Penalty}

The penalty function can be $\ell_2$ norm. In RLLE, Eq. (\ref{equation_LLE_linearReconstruct_2}) is regularized, with the regularization parameter $\gamma$, as \cite{winlaw2011robust}:
\begin{equation}
\begin{aligned}
& \underset{\{\widetilde{\b{w}}_i\}_{i=1}^n}{\text{minimize}}
& & \sum_{i=1}^n \widetilde{\b{w}}_i^\top \b{G}_i\, \widetilde{\b{w}}_i + \gamma \|\widetilde{\b{w}}_i\|_2^2, \\
& \text{subject to}
& & \b{1}^\top \widetilde{\b{w}}_i = 1, ~~~ \forall i \in \{1, \dots, n\}.
\end{aligned}
\end{equation}
The Lagrangian for this optimization is \cite{boyd2004convex}:
\begin{align*}
\mathcal{L} = \sum_{i=1}^n \widetilde{\b{w}}_i^\top \b{G}_i\, \widetilde{\b{w}}_i + \gamma \|\widetilde{\b{w}}_i\|_2^2 - \sum_{i=1}^n \lambda_i\, (\b{1}^\top \widetilde{\b{w}}_i - 1).
\end{align*}
Setting the derivative of Lagrangian to zero gives:
\begin{align}
&\mathbb{R}^{k} \ni \frac{\partial \mathcal{L}}{\partial \widetilde{\b{w}}_i} = 2 \b{G}_i \widetilde{\b{w}}_i + 2\gamma \widetilde{\b{w}}_i - \lambda_i \b{1} \overset{\text{set}}{=} \b{0}, \nonumber \\
&\implies \widetilde{\b{w}}_i = \frac{\lambda_i}{2} (\b{G}_i + \gamma \b{I})^{-1} \b{1}. \label{equation_RLLE_derivative_lagrangian_1} \\
&\mathbb{R} \ni \frac{\partial \mathcal{L}}{\partial \lambda} = \b{1}^\top \widetilde{\b{w}}_i - 1 \overset{\text{set}}{=} 0 \implies \b{1}^\top \widetilde{\b{w}}_i = 1. \label{equation_RLLE_derivative_lagrangian_2}
\end{align}
Using Eqs. (\ref{equation_RLLE_derivative_lagrangian_1}) and (\ref{equation_RLLE_derivative_lagrangian_2}), we have:
\begin{align*}
\frac{\lambda_i}{2} \b{1}^\top (\b{G}_i + \gamma \b{I})^{-1} \b{1} = 1 \implies \lambda_i = \frac{2}{\b{1}^\top (\b{G}_i + \gamma \b{I})^{-1} \b{1}}. 
\end{align*}
Hence:
\begin{align}
\widetilde{\b{w}}_i = \frac{\lambda_i}{2} (\b{G}_i + \gamma \b{I})^{-1} \b{1} = \frac{(\b{G}_i + \gamma \b{I})^{-1} \b{1}}{\b{1}^\top (\b{G}_i + \gamma \b{I})^{-1} \b{1}}.
\end{align}
Note that in addition to better handling of noise, this regularization solves the problem of possible singularity of the matrix $\b{G}_i$ by strengthening its main diagonal.  

\subsubsection{RLLE with Elastic-Net Penalty}

Another way of regularization for RLLE is using the elastic-net penalty function \cite{zou2005regularization} to also incorporate sparsity in the solution. 
This RLLE regularizes Eq. (\ref{equation_LLE_linearReconstruct_2}) as \cite{winlaw2011robust}:
\begin{equation}\label{equation_RLLE_linearReconstruct_elastic_net1}
\begin{aligned}
& \underset{\{\widetilde{\b{w}}_i\}_{i=1}^n}{\text{minimize}}
& & \sum_{i=1}^n \widetilde{\b{w}}_i^\top \b{G}_i\, \widetilde{\b{w}}_i +\! \gamma (\alpha \|\widetilde{\b{w}}_i\|_2^2\! +\! (1-\alpha) \|\widetilde{\b{w}}_i\|_1), \\
& \text{subject to}
& & \b{1}^\top \widetilde{\b{w}}_i = 1, ~~~ \forall i \in \{1, \dots, n\},
\end{aligned}
\end{equation}
where $\alpha \in [0,1]$. Note that $(\alpha \|\widetilde{\b{w}}_i\|_2^2 + (1-\alpha) \|\widetilde{\b{w}}_i\|_1)$ is the elastic-net function \cite{zou2005regularization}. 
As $\ell_1$ norm, i.e. $\|\widetilde{\b{w}}_i\|_1 = \sum_{j=1}^k |\widetilde{w}_{ij}|$, is not differentiable, we use $\widetilde{w}_{ij} := \widetilde{w}_{ij,+} - \widetilde{w}_{ij,-}$ where:
\begin{align}
\left\{
    \begin{array}{ll}
        \widetilde{w}_{ij,+} := |\widetilde{w}_{ij}|, \widetilde{w}_{ij,-} := 0 & \mbox{if } \widetilde{w}_{ij} \geq 0, \\
        \widetilde{w}_{ij,+} := 0, \widetilde{w}_{ij,-} := -|\widetilde{w}_{ij}| & \mbox{if } \widetilde{w}_{ij} \geq 0.
    \end{array}
\right.
\end{align}
Hence $|\widetilde{w}_{ij}| := \widetilde{w}_{ij,+} + \widetilde{w}_{ij,-}$. 
We define 
$\mathbb{R}^{k} \ni \widetilde{\b{w}}_{i,+} := [\widetilde{w}_{i1,+}, \dots, \widetilde{w}_{ik,+}]^\top$ 
and
$\mathbb{R}^{k} \ni \widetilde{\b{w}}_{i,-} := [\widetilde{w}_{i1,-}, \dots, \widetilde{w}_{ik,-}]^\top$ 
and
$\mathbb{R}^{2k} \ni \widetilde{\b{w}}^*_i := [\widetilde{\b{w}}_{i,+}^\top, \widetilde{\b{w}}_{i,-}^\top]^\top$ 
and
$\mathbb{R}^{d \times 2k} \ni \b{X}^*_i := [\b{X}_i, -\b{X}_i]$
and 
$\mathbb{R}^{2k \times 2k} \ni \b{G}^*_i := (\b{x}_i \b{1}_{2k\times 1}^\top - \b{X}^*_i)^\top (\b{x}_i \b{1}_{2k\times 1}^\top - \b{X}^*_i)$. 
Eq. (\ref{equation_RLLE_linearReconstruct_elastic_net1}) can be restated as:
\begin{equation}\label{equation_RLLE_linearReconstruct_elastic_net2}
\begin{aligned}
& \underset{\{\widetilde{\b{w}}^*_i\}_{i=1}^n}{\text{minimize}}
& & \sum_{i=1}^n \widetilde{\b{w}}_i^{*\top} \b{G}_i^*\, \widetilde{\b{w}}_i^* +\! \gamma (1-\alpha) \b{1}_{2k \times 1}^\top \widetilde{\b{w}}_i^*, \\
& \text{subject to}
& & \b{1}_{k \times 1}^\top \widetilde{\b{w}}_{i,+}^* - \b{1}_{k \times 1}^\top \widetilde{\b{w}}_{i,-}^* = 1, \\
& & & \widetilde{\b{w}}_{i}^* \succeq 0, ~~~ \forall i \in \{1, \dots, n\}.
\end{aligned}
\end{equation}
This optimization problem can be solved by sequential quadratic programming \cite{boggs1995sequential}. 

\section{Fusion of LLE with Other Manifold Learning Methods}\label{section_LLE_fusion_with_others}

\subsection{LLE with Geodesic Distances: Fusion of LLE with Isomap}

ISOLLE \cite{varini2005isolle} fuses LLE and Isomap \cite{tenenbaum2000global,ghojogh2020multidimensional}. Although LLE is a nonlinear manifold learning method, its $k$NN construction is linear because of usage of Euclidean distance. ISOLLE uses geodesic distance, which is alo used in Isomap, in the LLE method. 

The \textit{geodesic distance} is the length of shortest path between two points on the possibly curvy manifold. 
It is ideal to use the geodesic distance; however, calculation of the geodesic distance is very difficult because it requires traversing from a point to another point on the manifold. This calculation requires differential geometry and Riemannian manifold calculations \cite{aubin2001course}.   
Therefore, ISOLLE approximates the geodesic distance by piece-wise Euclidean distances. It finds the $k$-Nearest Neighbors ($k$NN) graph of dataset. Then, the shortest path between two points, through their neighbors, is found using a shortest-path algorithm such as the Dijkstra algorithm or the Floyd-Warshal algorithm \cite{cormen2009introduction}. A sklearn function in python for this is ``graph\_shortest\_path'' from the package ``sklearn.utils.graph\_shortest\_path''. 
The approximated geodesic distance can be formulated as \cite{bengio2004out}:
\begin{equation}\label{equation_geodesic_distance_matrix}
\b{D}^{(g)}_{ij} := \min_{\b{r}} \sum_{i = 2}^{l} \|\b{r}_i - \b{r}_{i+1}\|_2,
\end{equation}
where $l \geq 2$ is the length of sequence of points $\b{r}_i \in \{\b{x}_i\}_{i=1}^n$ and $\b{D}^{(g)}_{ij}$ denotes the $(i,j)$-th element of the geodesic distance matrix $\b{D}^{(g)} \in \mathbb{R}^{n \times n}$. 
For more information on geodesic distance, refer to \cite{ghojogh2020multidimensional}. 

ISOLLE makes use of geodesic distance matrix $\b{D}^{(g)}$, rather than the Euclidean distance matrix $\b{D}$, for construction of the $k$NN graph. The rest of ISOLLE is the same as in LLE. 

\subsection{Fusion of LLE with PCA}

LLE is fused with Principal Component Analysis (PCA) \cite{ghojogh2019unsupervised} in the LLE-guided PCA (LLE-PCA) \cite{jiang2018robust}. 
We denote centered data by $\mathbb{R}^{d \times n} \ni \breve{\b{X}} := \b{X} \b{H}$ where $\b{H} := \b{I} - (1/n) \b{1}\b{1}^\top$ is the centering matrix. 
PCA subspace can be found by Singular Value Decomposition (SVD) on the reconstructed data $\widehat{\b{X}}$, i.e., $\widehat{\b{X}} = \b{U}\b{\Sigma}\b{V}$. According to orthogonality of matrices in SVD, we have: $\b{U}^\top \b{U} = \b{I}$ and $\b{V} \b{U}^\top = \b{I}$. Minimization of reconstruction error is \cite{ghojogh2019unsupervised}:
\begin{equation}\label{equation_PCA_reconstruction}
\begin{aligned}
& \underset{\b{U},\b{\Sigma},\b{V}}{\text{minimize}}
& & ||\breve{\b{X}} - \widehat{\b{X}}||_F^2 = ||\breve{\b{X}} - \b{U}\b{\Sigma}\b{V}||_F^2, \\
& \text{subject to}
& & \b{U}^\top \b{U} = \b{I}, \\
& & & \b{V} \b{V}^\top = \b{I}.
\end{aligned}
\end{equation}
We absorb $\b{U}$ and $\b{\Sigma}$ to have $\widehat{\b{X}} = \b{U}\b{\Sigma}\b{V} = \b{R} \b{V}$ where $\b{R} = \b{U}\b{\Sigma}$. 
The embedded data or the projected data into the $p$-dimensional embedding space is $\mathbb{R}^{p \times n} \ni \b{Y}^\top := \b{U}^\top\breve{\b{X}}$ where $\b{U} \in \mathbb{R}^{d \times p}$ is the projection matrix. Therefore, the reconstructed data are $\widehat{\b{X}} = \b{U} \b{U}^\top \breve{\b{X}} = \b{U} \b{Y}^\top$. In summary, up to scale of singular values, we can consider the equality of $\b{RV}$ and $\b{U}\b{Y}^\top$. Hence, up to scale, we have $\b{Y} = \b{V}$ and $||\breve{\b{X}} - \b{U}\b{\Sigma}\b{V}||_F^2 = ||\breve{\b{X}} - \b{R}\b{Y}^\top||_F^2$. Note that according to the first constraint in Eq. (\ref{equation_LLE_linearEmbedding_2}), we have $\b{V} \b{V}^\top = \b{Y}^\top \b{Y} = \b{I}$ up to scale so the second constraint in Eq. (\ref{equation_PCA_reconstruction}) is automatically satisfied. To sum up, Eq. (\ref{equation_PCA_reconstruction}) is restated to:
\begin{equation}
\begin{aligned}
& \underset{\b{R}}{\text{minimize}}
& & ||\breve{\b{X}} - \b{R}\b{Y}^\top||_F^2.
\end{aligned}
\end{equation}
The Lagrangian of this optimization is \cite{boyd2004convex}:
\begin{align}
&\mathcal{L} = \frac{\partial ||\breve{\b{X}} - \b{R}\b{Y}^\top||_F^2}{\partial \b{R}} = 2 (\breve{\b{X}} - \b{R}\b{Y}^\top) \b{Y} \overset{\text{set}}{=} \b{0} \nonumber \\
&\implies \breve{\b{X}} - \b{R}\b{Y}^\top = \b{0} \implies \b{R} = \breve{\b{X}}(\b{Y}^\top)^{-1} \overset{(a)}{=} \breve{\b{X}} \b{Y}, \label{equation_LLE_PCA_R}
\end{align}
where $(a)$ is because $\b{Y}$ is an orthogonal matrix as we had $\b{Y}^\top \b{Y} = \b{I}$.
LLE-PCA \cite{jiang2018robust} centers data first. Then, it applies LLE to data for finding the embedding $\b{Y} \in \mathbb{R}^{n \times p}$. Then, it projects data onto the PCA subspace:
\begin{align}
\mathbb{R}^{d \times n} \ni \b{Y}_\text{LLE-PCA} := \b{R} \b{Y}^\top \overset{(\ref{equation_LLE_PCA_R})}{=} \breve{\b{X}} \b{Y} \b{Y}^\top,
\end{align}
stacked column-wise. Considering merely the first $p$ rows gives us the $p$-dimensional embedding $\b{Y}_\text{LLE-PCA} \in \mathbb{R}^{p \times n}$. 

\subsection{Fusion of LLE with FDA (or LDA)}

Unified LLE and Linear Discriminant Analysis Algorithm (ULLELDA) \cite{zhang2004unified} fuses LLE and FDA \cite{ghojogh2019fisher} (or LDA \cite{ghojogh2019linear}). 
First, it applies LLE on the high dimensional data to find the embeddings $\{\b{y}_i \in \mathbb{R}^p\}_{i=1}^n$ and the weights $\{w_{ij}\}_{i,j=1}^n$. These embeddings are projected onto the FDA subspace (see \cite{ghojogh2019fisher}) to have new embeddings $\{\b{z}_i \in \mathbb{R}^p\}_{i=1}^n$. 
The final embedding of $\b{x}_i$ is obtained as:
\begin{align}
\mathbb{R}^p \ni \b{y}_{i,\text{ULLELDA}} := \sum_{j=1}^n w_{ij}\, \b{z}_j.
\end{align}

\subsection{Fusion of LLE with FDA and Graph Embedding: Discriminant LLE}\label{section_discriminant_LLE}

Discriminant LLE (DLLE) \cite{li2008discriminant} is a supervised LLE method. Its overall idea is (I) to use the $k$NN of every point only from the points in the same class as the point and (II) to maximize and minimize the inter- and intra-class variances of data. 

DLLE uses $k$NN obtained from the neighbors of points from their classes and uses this $k$NN graph in optimization (\ref{equation_LLE_linearReconstruct}). 
Then, the weight matrix $\b{W} = [w_{ij}] \in \mathbb{R}^{n \times n}$ is obtained by Eq. (\ref{equaion_LLE_weight_and_weightHat}). 
A similarity matrix $\b{S} \in \mathbb{R}^{n \times n}$ is defined using the obtained weight matrix:
\begin{align}
\b{S}(i,j) := 
\left\{
    \begin{array}{ll}
        (\b{W} + \b{W}^\top - \b{W}^\top \b{W})(i,j) & \mbox{if } i = j, \\
        0 & \mbox{Otherwise,}
    \end{array}
\right.
\end{align}
inspired by graph embedding \cite{yan2005graph}. 
It also finds a $k$NN graph by considering the neighbors of a point from the different classes than the class of point. It defines a dissimilarity (or between-class) matrix $\b{B} \in \mathbb{R}^{n \times n}$ by:
\begin{align}
\b{B}(i,j) := 
\left\{
    \begin{array}{ll}
        1/k & \mbox{if } c_i \neq c_j, \\
        0 & \mbox{Otherwise.}
    \end{array}
\right.
\end{align}
Let the Laplacian matrices of $\b{S}$ and $\b{B}$ be denoted by $\b{L}_S$ and $\b{L}_B$, respectively. 
DLLE finds a projection matrix $\b{U}$ for maximizing and minimizing the inter- and intra-class variances:
\begin{equation}
\begin{aligned}
& \underset{\b{U}}{\text{maximize}}
& & \frac{\textbf{tr}(\b{U}^\top \b{X} \b{L}_B \b{X}^\top \b{U})}{\textbf{tr}(\b{U}^\top \b{X} \b{L}_S \b{X}^\top \b{U})},
\end{aligned}
\end{equation}
which is a Rayleigh-Ritz quotient \cite{ghojogh2019unsupervised} whose solution is a generalized eigenvalue problem $(\b{X} \b{L}_B \b{X}^\top, \b{X} \b{L}_S \b{X}^\top)$ \cite{ghojogh2019eigenvalue}. This optimization is inspired by Fisher discriminant analysis \cite{ghojogh2019fisher}. 

\subsection{Fusion of LLE with Isotop}

The paper \cite{lee2003locally} fuses LLE and Isotop \cite{lee2002nonlinear}. It first applied LLE on data to find $p$-dimensional embeddings.  Then, competitive learning \cite{ahalt1990competitive} is used for vector quantization of the embeddings. Some prototypes, as the final embeddings, are initialized. Afterwards, some random points are randomly drawn from Gaussian distributions. The prototypes, close to the random Gaussian points, are updated using a rule found in \cite{lee2003locally}. For the sake of brevity, we do not cover all details of this method in this paper.

\section{Weighted Locally Linear Embedding}\label{section_weighted_LLE}

Some works have been done on weighting the distances, reconstruction weights, or the embedding in LLE. In the following, we explain these works briefly. 

\subsection{Weighted LLE for Deformed Distributed Data}

Weighted LLE \cite{pan2009weighted} improves LLE especially if the distribution of data is deflated in the sense that it is much different from Gaussian distribution. 
They make use of a weighted distance defined as \cite{zhou2006improving}:
\begin{align}\label{equation_distance_deformed}
\text{dist}(\b{x}_i, \b{x}_j) &:= \frac{\|\b{x}_i - \b{x}_j\|_2}{a_i + b_i \frac{(\b{x}_i - \b{x}_j)^\top \b{\tau}_i}{\|\b{x}_i - \b{x}_j\|_2}} = \frac{\|\b{x}_i - \b{x}_j\|_2}{(a_i + b_j \cos \theta)},
\end{align}
where $\b{v}_{ij} := \b{x}_{ij} - \b{x}_i$ is calculated using the $k$NN by the Euclidean distance and then \cite{pan2009weighted}:
\begin{align}
&\b{\tau}_i := \frac{\b{g}_i}{\|\b{g}_i\|_2}, \quad a_i := \frac{l_i}{c_2}, \quad b_i := \frac{\|\b{g}_i\|_2}{c_1}, \\
& \b{g}_i := \frac{1}{k} \sum_{j=1}^k \b{v}_{ij}, \quad l_i := \frac{1}{k} \sum_{j=1}^k \|\b{v}_{ij}\|_2, \\
& c_1 = \sqrt{2} \frac{\Gamma((d+1)/2)}{\Gamma(d/2)\, d}, \quad c_2 = \sqrt{2} \frac{\Gamma((d+1)/2)}{\Gamma(d/2)},
\end{align}
where $\Gamma$ is the Gamma function and $d$ is the dimensionality of input space. 
Using Eq. (\ref{equation_distance_deformed}) as the distance rather than the Euclidean distance, we find the $k$NN graph. In formulation of LLE, this obtained $k$NN is used and the rest of algorithm is the same as in LLE. 

\subsection{Weighted LLE Using Probability of Occurrence}

There is a weighted LLE method using probability of occurrence  \cite{mekuz2005face} which is also applied in the field of face recognition. Assume data have a probability distribution; for example, a mixture distribution can be fitted to data using the expectation maximization algorithm \cite{ghojogh2019fitting}. Let the probability of occurrence of data point $\b{x}_i$ be $p_i$. The distance used in this weighted LLE is weighted by the probability of occurrence:
\begin{align}
\text{dist}^2(\b{x}_i, \b{x}_j) := \frac{\|\b{x}_i - \b{x}_j\|_2^2}{p_i}. 
\end{align}
Note that this weighting increases the distance of a point from its neighbors if its probability is low. This makes sense because an outlier or anomaly should be considered farther from other normal points. This makes LLE more robust to outliers. 

Using this weighted distance rather than the Euclidean distance, the $k$NN graph is calculated. Moreover, the Gram matrix, Eq. (\ref{equation_G}), is weighted by the probabilities of occurrence. If $\b{G}_i(a,b)$ denotes the $(a,b)$-th element of $\b{G}_i$, it is weighted as:
\begin{align}
\b{G}_i(a,b) := \sqrt{p_i\, p_j}\, \b{G}_i(a,b). 
\end{align}
The rest of algorithm is the same as in LLE. 

\subsection{Supervised LLE by Adjusting Weights}\label{section_adjusted_weights_LLE}

There is a supervised LLE method making use of labels to adjust the weights \cite{he2019nonlinear}. The obtained weights, by Eq. (\ref{equation_w_tilde_solution}), in LLE are weighted using the class labels. If two points are in the same class, the reconstruction weight between them is strengthened because they are similar (in the same class); otherwise, the weight is decreased:
\begin{align}
\widetilde{w}_{ij} \gets 
\left\{
    \begin{array}{ll}
        \widetilde{w}_{ij} + \delta & \mbox{if } c_i = c_j, \\
        \widetilde{w}_{ij} - \delta & \mbox{Otherwise.}
    \end{array}
\right.
\end{align}

\subsection{Modified Locally Linear Embedding}

Modified LLE (MLLE) \cite{zhang2007mlle} modifies or adjusts the reconstruction weights. It defines some new weights as:
\begin{align}
\mathbb{R}^k \ni \widetilde{\b{w}}_i^{(l)} = (1 - \alpha_i)\, \widetilde{\b{w}}_i + \b{V}_i\, \b{J}_i(:,l), 
\end{align}
for $l \in \{1, \dots, s_i\}$, where $\b{V}_i \in \mathbb{R}^{k \times s_i}$ is the matrix containing the $s_i$ smallest right singular vectors of $\b{G}_i$, $\alpha_i := (1 / \sqrt{s}_i) \|\b{v}_i\|_2$, $\b{v}_i := \b{V}_i^\top \b{1}_{k \times 1} \in \mathbb{R}^{s_i}$, and $\b{J}_i$ is a Householder matrix \cite{householder1953principles} satisfying $\b{H}_i \b{V}_i^\top \b{1}_{k \times 1} = \alpha_i \b{1}_{s_i \times 1}$. 
MLLE uses $\widetilde{\b{w}}_i^{(l)}$ rather than $\widetilde{\b{w}}_i$ in Eq. (\ref{equaion_LLE_weight_and_weightHat}) to have $w_{ij}^{(l)}$. 
This method slightly modifies the objective in optimization (\ref{equation_LLE_linearEmbedding}):
\begin{equation}
\begin{aligned}
& \underset{\b{Y}}{\text{minimize}}
& & \sum_{i=1}^n \sum_{l=1}^{s_i} \Big|\Big|\b{y}_i - \sum_{j=1}^n w_{ij}^{(l)} \b{y}_j\Big|\Big|_2^2,
\end{aligned}
\end{equation}
with the constraints in Eq. (\ref{equation_LLE_linearEmbedding}). 
The rest of algorithm is similarly solved as in LLE but with this modified objective function.

\subsection{Iterative Locally Linear Embedding}

Iterative LLE \cite{kong2012iterative} is a LLE-based method which has made several modifications to LLE. First, it restricts the weights to be non-negative. Hence, it changes Eq. (\ref{equation_LLE_linearReconstruct}) to:
\begin{equation}\label{equation_iterative_LLE_linearReconstruct}
\begin{aligned}
& \underset{\widetilde{\b{W}}}{\text{minimize}}
& & \varepsilon(\widetilde{\b{W}}) := \sum_{i=1}^n \Big|\Big|\b{x}_i - \sum_{j=1}^k \widetilde{w}_{ij} \b{x}_{ij}\Big|\Big|_2^2, \\
& \text{subject to}
& &\widetilde{w}_{ij} \geq 0, ~~~ \forall i \in \{1, \dots, n\}.
\end{aligned}
\end{equation}
Moreover, iterative LLE adjusts and weights the embedding $\b{Y}$ by including the diagonal degree matrix $\b{D} \in \mathbb{R}^{n \times n}$ to the constraint in Eq. (\ref{equation_LLE_linearEmbedding_2}):
\begin{equation}\label{equation_iterative_LLE_linearEmbedding_2}
\begin{aligned}
& \underset{\b{Y}}{\text{minimize}}
& & \textbf{tr}(\b{Y}^\top\b{M}\b{Y}), \\
& \text{subject to}
& & \frac{1}{n} \b{Y}^\top \b{D} \b{Y} = \b{I}, \\
& & & \b{Y}^\top \b{1} = \b{0},
\end{aligned}
\end{equation}
which has some relations with the spectral embedding \cite{chan1994spectral} and Laplacian embedding \cite{belkin2003laplacian}. 
The iterative LLE \cite{kong2012iterative} also iterates between the solutions of Eqs. (\ref{equation_iterative_LLE_linearReconstruct}) and (\ref{equation_iterative_LLE_linearEmbedding_2}) to improve the embedding of LLE.








\section{Conclusion}\label{section_conclusion}

In this tutorial and survey paper, we explain LLE and its variants. 
We explained that the main idea of LLE is piece-wise local fitting of manifold to hopefully unfold the overall manifold. 
The quality of this unfolding depends on the parameters of LLE which can be tuned by some methods introduced in this paper.
The materials which were covered in this paper are LLE, inverse LLE, feature fusion with LLE, kernel LLE, out-of-sample embedding (using linear reconstruction and eigenfunctions), incremental LLE for streaming data, landmark LLE (using the Nystrom approximation and locally linear landmarks), parameter selection of the number of neighbors (using residual variance, Procrustes statistics, preservation neighborhood error, and local neighborhood selection), supervised and semi-supervised LLE (including SLLE, enhanced SLLE, SLLE projection, probabilistic SLLE, semi-supervised LLE), robust LLE, fusion of LLE with other manifold learning methods (including Isomap, PCA, FDA, discriminant LLE, and Isotop), weighted LLE (for deformed distributed data, using probability of occurrence, by adjusting weights, modified LLE, and iterative LLE). 
Some other LLE methods were not covered in this paper. For example, Locally Linear Image Structure Embedding (LLISE) \cite{ghojogh2019locally} formulates LLE using the Structural Similarity Index (SSIM) \cite{wang2004image} for image structure manifold learning. 
Moreover, note that there is an official MATLAB library for LLE which can be found in \cite{roweis2020LLEweb}. 

\bibliography{References}
\bibliographystyle{icml2016}

\end{document}